\documentclass[10pt, conference, compsocconf]{IEEEtran}
\IEEEoverridecommandlockouts
\renewcommand{\baselinestretch}{0.855}
\usepackage[lofdepth,lotdepth]{subfig}
\usepackage[bookmarks=false]{hyperref}
\usepackage{cite}
\usepackage[cmex10]{amsmath}
\usepackage{amsthm}
\usepackage{amssymb,amsfonts}
\newtheorem{definition}{Definition}

\newtheorem{theorem}{Theorem}

\newtheorem{lemma}{Lemma}
\usepackage{changepage}
\usepackage[noend]{algorithmic}
\usepackage{algorithm2e}
\usepackage{multirow}
\usepackage{xcolor}
\usepackage{color}
\usepackage{times}
\usepackage{mathptmx}
\usepackage{mathtools}
\usepackage{mathrsfs}
\usepackage{enumerate}
\usepackage{amssymb}
\usepackage{hyperref}
\usepackage{graphicx}
\usepackage{textcomp}
\usepackage{xcolor}

\DeclareMathAlphabet{\mathcal}{OMS}{cmsy}{m}{n}

\newcommand{\mc}[1]{\textcolor{blue}{#1}} 
\newcommand{\nop}[1]{}

\newcommand{\wanglj}[1]{{\textbf{\color{orange} #1}}}
\newcommand{\figwidthone}{0.47} 

\def\BibTeX{{\rm B\kern-.05em{\sc i\kern-.025em b}\kern-.08em
    T\kern-.1667em\lower.7ex\hbox{E}\kern-.125emX}}

\begin{document}

\title{Exact and Consistent Interpretation of Piecewise Linear Models Hidden behind APIs: A Closed Form Solution}

\author{
\IEEEauthorblockN{Zicun Cong$^{*}$, Lingyang Chu$^{\S}$, Lanjun Wang$^{\S}$, Xia Hu$^{*}$, Jian Pei$^{*}$}
    \IEEEauthorblockA{
        $^{*}$Simon Fraser University, Burnaby, Canada \\
        $^{\S}$ Huawei Technologies Canada Co., Ltd., Burnaby, Canada\\
        Emails: \{zcong, huxiah, jpei\}@sfu.ca, \{lingyang.chu1, lanjun.wang\}@huawei.com
    }
}

\maketitle

\begin{abstract}
More and more AI services are provided through APIs on cloud where predictive models are hidden behind APIs. To build trust with users and reduce potential application risk, it is important to interpret how such predictive models hidden behind APIs make their decisions.
The biggest challenge of interpreting such predictions is that no access to model parameters or training data is available. 
Existing works interpret the predictions of a model hidden behind an API by heuristically probing the response of the API with perturbed input instances.
However, these methods do not provide any guarantee on the exactness and consistency of their interpretations.
In this paper, we propose an elegant closed form solution named \texttt{OpenAPI} to compute exact and consistent interpretations for the family of Piecewise Linear Models (PLM), which includes many popular classification models.
The major idea is to first construct a set of overdetermined linear equation systems with a small set of perturbed instances and the predictions made by the model on those instances. Then, we solve the equation systems to identify the decision features that are responsible for the prediction on an input instance. Our extensive experiments clearly demonstrate the exactness and consistency of our method.

\end{abstract}

\nop{
}

\section{Introduction}
More and more machine learning systems are deployed as cloud services to make important decisions routinely in many application areas, such as medicine, biology, financial business, and autonomous vehicles~\cite{Goodfellow-et-al-2016}.
As more and more decisions in both number and importance are made, the demand on clearly interpreting these decision making processes is becoming ever stronger~\cite{goodman2016european}.
Accurately and reliably interpreting these decision making processes is the key to many essential tasks, such as detecting model failures~\cite{agrawal2016analyzing}, building trust with public users~\cite{ribeiro2016should}, and preventing models from unfairness~\cite{zemel2013learning}.

Many methods have been proposed to interpret a machine learning model (see Section~\ref{sec:rw} for a brief review).
Most of those methods are applicable only when they have full access to training data and model parameters.
Unfortunately,  they cannot interpret decisions made by machine learning models encapsulated by cloud services, because service providers always protect and hide their sensitive training data and predictive models as top commercial secrets~\cite{tramer2016stealing}. 
More often than not, only application program interfaces (APIs) are provided to public users.


The local perturbation methods~\cite{ribeiro2016should, ribeiro2018anchors, chen2017zoo, fong2017interpretable} are developed to interpret predictive models that only APIs but no training data or model parameters are known. The major idea is to identify the decision features of a model by analyzing the predictions on a set of perturbed instances that are generated by perturbing (i.e., slightly modifying) the features of an instance to be interpreted. However, since the space of possible feature perturbations is exponentially large with respect to the dimensionality of the feature space, those methods can only heuristically search a tiny portion of the perturbation space in a reasonable amount of time. There is no guarantee that the decision features found are exactly the decision features of the model to be interpreted~\cite{chu2018exact}. The reliability of the explanations still remains an unsolved big challenge~\cite{du2018techniques}. Poor interpretations may mislead users in many scenarios~\cite{ribeiro2018anchors}.

\emph{Can we compute exact and consistent interpretations of decisions made by predictive models hidden behind cloud service APIs?}
In this paper, affirmatively we provide an elegant closed form solution for the family of piecewise linear models. 
Here, a \textbf{piecewise linear model (PLM)} is a non-linear classification model whose classification function is a piecewise linear function.
In other words, a PLM consists of many locally linear regions, such that all instances in the same locally linear region are classified by the same locally linear classifier~\cite{chu2018exact}.
The family of PLM hosts many popular classification models, such as logistic model trees~\cite{landwehr2005logistic, sumner2005speeding} and the entire family of piecewise linear neural networks~\cite{chu2018exact} that use MaxOut~\cite{goodfellow2013maxout} or ReLU family~\cite{glorot2011deep, nair2010rectified,he2015delving} as activation functions. For example, the implementations of the AlexNet~\cite{krizhevsky2012imagenet}, the VGG Net~\cite{simonyan2014very}, and the ResNet~\cite{he2016deep} all belong to the family of PLM.
Due to the extensive applications~\cite{lecun2015deep} and tremendous practical successes~\cite{krizhevsky2012imagenet} of piecewise linear models, exact interpretations of piecewise linear models hidden behind APIs are greatly useful in many critical application tasks.


%


Our major technical contribution in this paper is to develop \texttt{OpenAPI},  a method to exactly interpret the predictions made by a PLM model behind an API without accessing model parameters or training data.
Specifically, \texttt{OpenAPI} identifies the decision feature, which is a vector
showing the importance degree of each feature, for an  instance to be interpreted by finding the closed form solutions to a set of overdetermined linear equation systems. The equation systems are simply constructed using a small set of sampled instances.
We prove that the decision features identified by \texttt{OpenAPI} are exactly the decision features of the PLM with probability $1$.
Our interpretations are consistent within each locally linear region, because
\texttt{OpenAPI} accurately identifies the decision features of a locally linear classifier, and those decision features are the same for all instances in the same locally linear region.
We conduct extensive experiments to demonstrate the exactness and consistency of our interpretations superior to five state-of-the-art interpretation methods~\cite{ribeiro2016should, chen2017zoo, simonyan2013deep, shrikumar2016not, sundararajan2017axiomatic}.

The rest of the paper is organized as follows. 
We review related works in Section~\ref{sec:rw}, and formulate our problem in Section~\ref{sec:prob}. We develop \texttt{OpenAPI} in Section~\ref{sec:oam}, and present the experimental results in Section~\ref{sec:exp}.  We conclude the paper in Section~\ref{sec:con}.



\section{Related Works}
\label{sec:rw}
How to interpret decisions made by predictive models is an emerging and challenging problem. There are four major groups of methods, briefly reviewed here.


First, the \textbf{instance attribution methods} find the training instances that significantly influence the prediction on an instance to be interpreted.
%
%
Wojnowicz~\emph{et~al.}~\cite{wojnowicz2016influence} used influence sketching to identify the training instances that strongly affect the fit of a regression model by efficiently estimating Cook's distance~\cite{cook1977detection}.
Koh~\emph{et~al.}~\cite{koh2017understanding} proposed influence functions to trace the prediction of a model and identify the training instances that are the most responsible for the prediction.
Bien~\emph{et~al.}~\cite{bien2011prototype} proposed a prototype selection algorithm to find a small set of representative training instances that capture the full variability of a class without confusing with the other classes.
Zhou~\emph{et~al.}~\cite{zhou2017interpreting} 
identified the instances that dominate the activation of the same hidden neuron of a convolutional neural network, and used the common labeled concept of those instances to interpret the semantic of the hidden neuron.


The instance attribution methods rely heavily on training data, which, unfortunately, is unavailable in most of the practical scenarios where only the APIs of the predictive models are provided.\nop{\wanglj{tbc}}









Second, the \textbf{model intimating methods}
train a self-explaining model to intimate the predictions of a deep neural network~\cite{ba2014deep, che2015distilling, hinton2015distilling}.
Hinton~\emph{et~al.}~\cite{hinton2015distilling} proposed to distill the knowledge of a large neural network by training a smaller network to imitate the predictions of the large network.
To make the distilled knowledge easier to understand, Frosst~\emph{et~al.}~\cite{frosst2017distilling} extended the distillation method~\cite{hinton2015distilling} by training a soft decision tree to mimic the predictions of a deep neural network.
Ba~\emph{et~al.}~\cite{ba2014deep} trained a shallow mimic network to distill the knowledge of one or more deep neural networks.
Wu~\emph{et~al.}~\cite{wu2017beyond} used a binary decision tree to mimic and regularize the prediction function of a deep time-series model.
Guo~\emph{et~al.}~\cite{guo2018explaining} trained a Dirichlet Process regression mixture model to approximate the decision boundary of the intimated model near an instance to be interpreted.


The model intimating methods produce understandable interpretations. 
They, however, cannot be directly applied to interpret models hidden behind APIs, because they cannot access training data to conduct mimic training.
Moreover, since a mimic model is not exactly the same as the intimated model, the interpretations may not exactly match the real behavior of the intimated model~\cite{chu2018exact}.








Third, the \textbf{gradient analysis methods}~\cite{simonyan2013deep, zhou2016learning, sundararajan2017axiomatic} find the important decision features for an instance to be interpreted by analyzing the gradient of the prediction score with respect to the instance.
Simonyan~\emph{et~al.}~\cite{simonyan2013deep} generated a class-saliency map and a class-representative image for each class of instances by computing the gradient of the class score with respect to an input instance.
Zhou~\emph{et~al.}~\cite{zhou2016learning} proposed CAM to find discriminative instance regions for each class using the global average pooling in Convolutional Neural Networks (CNN).
Selvaraju~\emph{et~al.}~\cite{selvaraju2016grad} generalized CAM~\cite{zhou2016learning} to Grad-CAM by identifying important regions of an image, i.e., a sub-matrix, and propagating class-specific gradients into the last convolutional layer of a CNN.
%
%
%
Smilkov~\emph{et~al.}~\cite{smilkov2017smoothgrad} proposed SmoothGrad to visually sharpen the gradient-based sensitivity map of an image to be interpreted.
Chu~\emph{et~al.}~\cite{chu2018exact} transformed a piecewise linear neural network into a set of locally linear classifiers, and interpreted the prediction on an input instance by analyzing the gradients of all neurons with respect to the instance.



The interpretations produced by the gradient analysis methods are faithful to the real behavior of the model to be interpreted. The computation of gradients, however,  requires full access to model parameters, which is usually not provided by the predictive models hidden behind APIs.



Last, the \textbf{local perturbation methods} interpret the behavior of a predictive model in a small neighborhood of the instance to be interpreted.
The key idea is to use a simple and interpretable model to analyze the predictions on a set of perturbed instances generated by perturbing  the features of the instance to be interpreted. 
Ribeiro~\emph{et~al.}~\cite{ribeiro2016should} proposed \texttt{LIME} to capture the decision features for an instance to be interpreted by training a linear model that fits the predictions on a sample of the perturbed instances.
They also proposed \texttt{Anchors}~\cite{ribeiro2018anchors}
to find the explanatory rules that dominate the predictions on a sample of perturbed instances.
%
%
Fong~\emph{et~al.}~\cite{fong2017interpretable} proposed to interpret the classification result of an image by finding the smallest pixel-deletion mask that causes the most significant drop of the prediction score. 

The local perturbation methods, on the one hand, generate interpretations easy to understand without accessing model parameters or training data. 
On the other hand, the interpretations may not be even correct, since the interpretation error is proportional to $f(\epsilon, n) + g(m)$, where the first component $f(\epsilon, n)$ represents the \emph{parameter related error}, $\epsilon$ being the perturbation distance and $n$ the number of perturbed instances, and the second component $g(m)$ is the \emph{approximate model related error} of the approximate model $m$.
Parameters not well selected may lead to a large error $f(\epsilon, n)$. The perturbation distances may be so large that the target model's behaviors on those perturbed instances are too complicated to be learned by a simple model. The \emph{approximate model related error} is due to the weaker approximation capabilities of simple models. For example, a linear model cannot exactly describe the non-linear behavior of a target model.

Although existing methods can decrease the errors in their interpretations using smaller neighborhoods, more perturbed instances, and better approximate models, the errors cannot be eliminated due to the following reason.
Different instances may have different applicable perturbation distances, that is, radii where the same interpretations still apply. The proper perturbation distance for an instance can be arbitrarily small, as the instance can be arbitrarily close to the boundary of the locally linear regions. 

Figure~\ref{fig:hardness} elaborates the subtlety. Suppose the 2-dimensional input space is separated into two regions by a PLM (the solid boundary). Each region has a unique linear classifier whose decision boundaries are the dashed lines. The two red solid boxes of the same size represent the neighborhoods of two instances, $A$ and $B$. As the neighborhood of $A$ completely falls into a class region of the PLM, the PLM behaves linearly there. Thus, the existing methods can obtain accurate interpretations for the prediction on $\mathbf{A}$ by applying a simple model to analyze the perturbed instances from the neighborhood. However, the neighborhood of $B$ overlaps the decision boundary and thus the PLM does not behave linearly in the neighborhood of $B$.  Consequently, the existing methods cannot find a simple model performing exactly the same as the PLM. 

The existing methods rely on a user defined perturbation distance. However, without accessing the parameters of a target model, it is impossible to find a universally applicable perturbation distance. One may wonder whether we can shrink the neighborhood size adaptively until the approximate models perfectly fit the perturbed instances. Unfortunately, the numerical optimization techniques used to train the approximate models, such as gradient descent, do not allow the current methods to reach the exact solutions~\cite{LR02}. 
\nop{Since the existing methods cannot guarantee the exactness of interpretations, users cannot tell whether an unexpected explanation is caused by the misbehavior of the model or by the limitations of the explanation methods~\cite{du2018techniques}.}
The fact that existing methods cannot guarantee  exactness of interpretations prevents them from being trusted by users. When the internal parameters of a target model are unavailable, users cannot verify the correctness of the interpretations. Therefore, users cannot tell whether an unexpected explanation is caused by the misbehavior of the model or by the limitations of the explanation methods~\cite{du2018techniques}.






In this paper, we develop \texttt{OpenAPI} to overcome the shortage. \texttt{OpenAPI} guarantees to find the exact decision features of the model to be interpreted with probability $1$, and thus leads to a significant advantage on producing exact and consistent interpretations for the PLMs hidden behind APIs.

\begin{figure}
    \centering
    \includegraphics[scale=0.08]{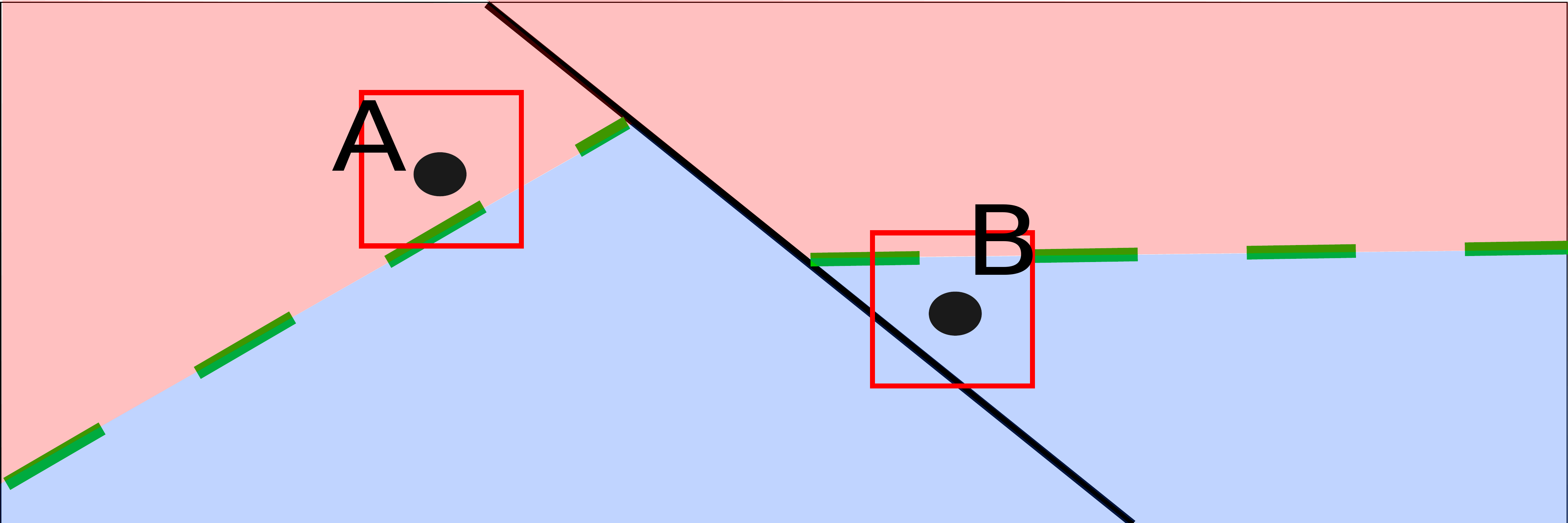}
    \caption{The hardness of getting exact interpretations for PLMs. }
    \label{fig:hardness}
\end{figure}




\section{Problem Definition}
\label{sec:prob}
Denote by $\mathcal{N}$ a piecewise linear model (PLM), and by $\mathbf{x}\in\mathcal{X}$ an \textbf{input instance} of $\mathcal{N}$, where $\mathcal{X}\in\mathbb{R}^d$ is a $d$-dimensional input space. $\mathbf{x}$ is also called an \textbf{instance} for short.
The \textbf{output} of $\mathcal{N}$ is $\mathbf{y}\in\mathcal{Y}$, where $\mathcal{Y}\in\mathbb{R}^C$ is a $C$-dimensional output space, and $C$ is the total number of classes.

A PLM works as a piecewise linear classification function $F:\mathcal{X}\rightarrow\mathcal{Y}$ that maps an input $\mathbf{x}\in\mathcal{X}$ to an output $\mathbf{y}\in\mathcal{Y}$. 
Denote by $\mathcal{X}_k\subset \mathcal{X}$ the $k$-th \textbf{locally linear region} of $\mathcal{X}$ such that $F(\cdot)$ operates as a locally linear classifier in $\mathcal{X}_k$. 

Denote by $K$ the number of all locally linear regions of $F(\cdot)$. Then, $\{\mathcal{X}_1, \ldots, \mathcal{X}_K\}$ forms a partition of $\mathcal{X}$, that is, $\cup_{k=1}^K \mathcal{X}_k = \mathcal{X}$, and $\mathcal{X}_k \cap \mathcal{X}_h = \emptyset$ when $k\neq h$. 
For common PLMs such as logistic model trees~\cite{landwehr2005logistic, sumner2005speeding} and piecewise linear neural networks~\cite{montufar2014number, chu2018exact, pascanu2013number}, the number of locally linear regions is finite.


Without loss of generality, we write the locally linear classifier in $\mathcal{X}_k$ as 
\begin{equation}\nonumber
	\sigma(W_k^{\top}\mathbf{x} + \mathbf{b}_k),
\end{equation}
where $W_k\in \mathbb{R}^{d\times C}$ is a $d$-by-$C$ dimensional coefficient matrix of $\mathbf{x}\in\mathcal{X}_k$, $\mathbf{b}_k\in\mathbb{R}^C$ is a $C$-dimensional bias vector, and $\sigma(\cdot)$ is a probabilistic scoring function, which can be sigmoid and softmax for binary classification and multi-class classification, respectively.
Since the softmax function is the general form of the sigmoid function, we assume $\sigma(\cdot)$ to be the softmax function by default, and write the complete form of $F(\cdot)$ as follows.
\begin{equation}\nonumber
    F(\mathbf{x}) =
    \begin{cases}
        \sigma(W_1^{\top}\mathbf{x} + \mathbf{b}_1) & \text{if $\mathbf{x} \in \mathcal{X}_1$}\\
        \sigma(W_2^{\top}\mathbf{x} + \mathbf{b}_2) & \text{if $\mathbf{x} \in \mathcal{X}_2$}\\
        \quad\quad\quad\vdots \\
        \sigma(W_K^{\top}\mathbf{x} + \mathbf{b}_K) & \text{if $\mathbf{x} \in \mathcal{X}_K$}\\
    \end{cases}
\end{equation}

Given an input instance $\mathbf{x}$, without loss of generality, denote by $\mathcal{X}_k$ the locally linear region that contains $\mathbf{x}$, the classification on $\mathbf{x}$ is uniquely determined by the locally linear classifier $\sigma(W_k^{\top}\mathbf{x} + \mathbf{b}_k)$. 
For the sake of simplicity, we omit the subscript $k$ when it is clear from the context, and write the classification result of $\mathbf{x}$ as 
\begin{equation}\nonumber
	\mathbf{y}=\sigma(W^\top\mathbf{x} + \mathbf{b})
\end{equation}


Following a principled approach of interpreting a machine learning model~\cite{bishop2007pattern, ribeiro2016should, chen2017zoo}, 
we regard an interpretation on the classification result of an input instance $\mathbf{x}$ as the decision features that classify $\mathbf{x}$ as one class and distinguish $\mathbf{x}$ from the other $C-1$ classes.
The formal definition of decision features will be discussed in Section~\ref{sec:decisionf}.
Formally, we define the task to interpret PLMs hidden behind APIs as follows.

\begin{definition}
Given the API of a PLM $\mathcal{N}$ and an input instance $\mathbf{x}\in\mathcal{X}$, for each class $c\in\{1, \ldots, C\}$, our task is to identify the decision features of $\mathcal{N}$ that classify $\mathbf{x}$ as class $c$ and distinguish it from the other $C-1$ classes.
\end{definition}
















\section{Interpretation Methods}
\label{sec:oam}
In this section, we first introduce the decision features of a PLM in classifying an instance. Then, we illustrate a naive method to compute the decision features of a PLM under an ideal case. Last, as the ideal case may not always appear, we introduce the \texttt{OpenAPI} method that computes the exact decision features without using any training data or model parameters.


\subsection{Decision Features of a PLM}
\label{sec:decisionf}

Some existing methods~\cite{ancona2018towards} interpret model predictions by computing the partial derivatives of model outputs with respect to input features. The partial derivatives are used as importance weights of features. However, those methods  do not work well for PLMs hidden behind APIs. First, to reliably compute the exact partial derivatives, the internal parameters of PLMs are needed. Second, for all instances in the same locally linear region, the weights of the corresponding features have to be consistent, as they are classified by the same locally linear classifier\cite{chu2018exact}. However, the feature weights computed by the gradient-based methods are different for different input instances.

Based on the coefficient matrices of the locally linear classifiers, we propose a new way to interpret the predictions made by PLMs. Our proposed interpretation not only describes the behaviors of PLMs exactly but is also consistent for the predictions made by the same locally linear classifiers.

Consider the output $\mathbf{y}$ of a PLM $\mathcal{N}$ on an instance $\mathbf{x}$.  For any class $c\in\{1, \ldots, C\}$, the $c$-th entry of $\mathbf{y}$, denoted by $\mathbf{y}_c$, is the probability to predict $\mathbf{x}$ as class $c$.
Denote by $W_c\in\mathbb{R}^d$ the $c$-th column of $W$ and by $\mathbf{b}_c\in\mathbb{R}$ the $c$-th entry of $\mathbf{b}$, we can expand the locally linear classifier $\mathbf{y}=\sigma(W^\top\mathbf{x} + \mathbf{b})$ and write the $c$-th entry of $\mathbf{y}$ as $\mathbf{y}_c\varpropto e^{W_c^\top \mathbf{x} + \mathbf{b}_c}$.

Following the routine of interpreting conventional linear classifiers, such as Logistic Regression and linear SVM \cite{bishop2007pattern}, $W_c$ is the vector of weights for all features in predicting $\mathbf{x}$ as class $c$. The features with positive (negative) weights in $W_c$ support (oppose) to predict $\mathbf{x}$ as class $c$.

Denote by $W_{c'}$, $c'\neq c$, the vector of weights for all features in predicting $\mathbf{x}$ as class $c'$. The difference between $W_c$ and $W_{c'}$, $D_{c, c'}=W_c - W_{c'}$, identifies the features that classify $\mathbf{x}$ as class $c$ and distinguishes $\mathbf{x}$ from class $c'$. To be specific, as $\mathbf{y}_c / \mathbf{y}_{c'} \varpropto e^{(W_c - W_{c'})^\top \mathbf{x} + \mathbf{b}_c - \mathbf{b}_{c'}}$, the input features of positive values in $D_{c, c'}$ increase the confidence of the model on class $c$ over class $c'$, and vice versa. As a result, $D_{c, c'}$ defines the \textbf{decision boundary} between class $c$ and class $c'$, thus is exactly the decision features of binary classification PLMs.

For general multi-class classification PLMs (i.e., $C\geq 2$), we interpret their predictions in the way of one-vs-the-rest. We can identify the decision features that classify $\mathbf{x}$ as class $c$ and distinguish it from the other $C-1$ classes by the average of the vectors $D_{c, c'}$ for all $c'\in\{1, \ldots, C\}\setminus c$.
Since $D_{c, c}=\mathbf{0}$, we can write this average of vectors as
\begin{equation}
\label{eqn:dc}
	D_c = \frac{1}{C-1}\sum_{c'=1}^C D_{c, c'}
\end{equation}

Here, the \textbf{decision features} $D_c$ are a $d$-dimensional vector that contains the importance weight of each feature in classifying $\mathbf{x}$ as class $c$. 
A feature with a larger absolute weight in $D_c$ is more important than one with a smaller absolute weight in classifying $\mathbf{x}$ as class $c$.
In addition, the signs of the weights in $D_c$ indicate the directions of the influences of the features on the prediction. The features of positive weights in $D_c$ support the predictions of the model on the class $c$ over any other classes, and vice versa.
In other words, $D_c$ is the answer to interpreting why a PLM classifies an instance $\mathbf{x}$ as class $c$ instead of some other classes. 
As $D_c$ is computed solely from the coefficient matrices of the locally linear classifiers, for two instances $\mathbf{x}$ and $\mathbf{x}'$ in the same locally linear region, they have the same $D_c$. This property enables our method to provide consistent interpretations for predictions made on instances from the same locally linear regions.


We can easily compute $D_c$ when the model parameters of a PLM are given.
For example, $D_c$ can be easily extracted from the model parameters of the conventional PLMs such as logistic model trees~\cite{landwehr2005logistic, sumner2005speeding}.
For piecewise linear neural networks, there is also an existing method~\cite{chu2018exact} that computes $D_c$ in polynomial time when the model parameters are given. 
However, none of the above methods can be used to compute $D_c$ when model parameters are unavailable.

\subsection{A Naive Method}
\label{method:naive}
To use only the API of a PLM to compute $D_c$ without accessing any model parameters, in this subsection, we introduce a naive method by solving $C-1$ determined linear equation systems. In an ideal case, the solution is exactly the same as $D_c$.

Given a tuple $(\mathbf{x}, \mathbf{y})$ where $\mathbf{x}\in\mathcal{X}$ is an input instance and $\mathbf{y}= \sigma(W^\top\mathbf{x} + \mathbf{b})$ is the prediction on $\mathbf{x}$, 
our goal is to compute $D_c$ for $\mathbf{x}$ by computing the set $\{D_{c, c'}\}$ such that $c'\in\{1, \ldots, C\}\setminus c$.





For $c$ and $c'$, denote by $B_{c, c'}=\mathbf{b}_{c} - \mathbf{b}_{c'}$ the difference between bias vectors $\mathbf{b}_{c}$ and $\mathbf{b}_{c'}$.
By decomposing the softmax function $\sigma(\cdot)$ of the locally linear classifier $\mathbf{y}=\sigma(W^\top\mathbf{x}+\mathbf{b})$, we have
\begin{equation}\nonumber
	\frac{\mathbf{y}_{c}}{\mathbf{y}_{c'}} 
	= \frac{e^{W_{c}^\top \mathbf{x} + \mathbf{b}_{c}}}{e^{W_{c'}^\top \mathbf{x} + \mathbf{b}_{c'}}}
	= e^{{D_{c,c'}}^\top \mathbf{x} + B_{c,c'}},
\end{equation}
which can be transformed into the following linear equation
\begin{equation}
\label{eqn:lineareq}
	{D_{c,c'}}^\top \mathbf{x} + B_{c,c'} = \ln(\frac{\mathbf{y}_{c}}{\mathbf{y}_{c'}})
\end{equation}

Since $\mathbf{x}$, $\mathbf{y}_c$ and $\mathbf{y}_{c'}$ are known variables, Equation~\ref{eqn:lineareq} contains $d+1$ unknown variables, which are the entries of the $d$-dimensional vector $D_{c,c'}\in\mathbb{R}^d$ and the scalar $B_{c, c'}\in \mathbb{R}$.

Tuple $(D_{c,c'}, B_{c,c'})$ fully characterizes the behavior of a locally linear classifier $\mathbf{y}=\sigma(W^\top\mathbf{x}+\mathbf{b})$ in classifying classes $c$ and $c'$.
If two locally linear classifiers have exactly the same $(D_{c,c'}, B_{c,c'})$ for every pair $c$ and $c'$, they produce exactly the same output $\mathbf{y}$ for the same input instance $\mathbf{x}$.
As a result, we call $(D_{c,c'}, B_{c,c'})$ the \textbf{core parameters} of a locally linear classifier in classifying classes $c$ and $c'$.
The core parameters of the locally linear classifier for an instance $\mathbf{x}$ is also said to be the core parameters of $\mathbf{x}$ for short.

For any pair $c$ and $c'$, a \textbf{naive method} to compute the core parameters of $\mathbf{x}$ is to construct and solve a determined linear equation system, denoted by $\Omega_{d+1}^{c, c'}$, that consists of $d+1$ linearly independent linear equations with the same core parameters as $\mathbf{x}$.

Since we already obtain one of these linear equations from $(\mathbf{x}, \mathbf{y})$, we only need to build another $d$ linear equations by independently and uniformly sampling $d$ instances in the neighborhood of $\mathbf{x}$. A $d$-dimension \textbf{hypercube} with edge length $2r$ and $\mathbf{x}$ as the center is defined as $\{\mathbf{p}\ |\ \forall i\ |\mathbf{p}_i - \mathbf{x}_i| \leq r,\  \mathbf{p}\in\mathbb{R}^d\}$, where $\mathbf{x}_i$ is the $i$-th entry of $\mathbf{x}$. In this paper, the \textbf{neighborhood} of $\mathbf{x}$ refers to the hypercube centered at $\mathbf{x}$.
%
We will illustrate how to compute $r$ later in Algorithm~\ref{algo:interpret}.





Denote by $\mathbf{x}^i, i\in\{1, \ldots, d\}$, the $i$-th sampled instance in the neighborhood of $\mathbf{x}$, and by $\mathbf{y}^i$ the prediction on $\mathbf{x}^i$. Obviously, $\mathbf{y}^i$ can be easily obtained by feeding $\mathbf{x}^i$ into the API of a PLM.
Tuple $(\mathbf{x}^i, \mathbf{y}^i)$ is used to build the $i$-th linear equation of $\Omega_{d+1}^{c, c'}$ in the same way as Equation~\ref{eqn:lineareq}. 


In the \textbf{ideal case} where the core parameters of all sampled instances are the same as the core parameters of $\mathbf{x}$, all linear equations in $\Omega_{d+1}^{c, c'}$ are linear equations of the same core parameters as $\mathbf{x}$.
Therefore, we can write $\Omega_{d+1}^{c, c'}$ as 
\begin{equation}
\label{eqn:lineareqsys}
\Omega_{d+1}^{c,c'} =
    \begin{cases}
    	{D_{c,c'}}^\top \mathbf{x}^0 + B_{c,c'}  = \ln(\frac{\mathbf{y}^0_c}{\mathbf{y}^0_{c'}}) \\
	{D_{c,c'}}^\top \mathbf{x}^1 + B_{c,c'}  = \ln(\frac{\mathbf{y}^1_c}{\mathbf{y}^1_{c'}}) \\
	\quad\quad\quad\vdots \\
	{D_{c,c'}}^\top \mathbf{x}^d + B_{c,c'}  = \ln(\frac{\mathbf{y}^d_c}{\mathbf{y}^d_{c'}}) \\
     \end{cases}
\end{equation}
where $(D_{c,c'}, B_{c,c'})$ is the core parameters of $\mathbf{x}$ in classifying classes $c$ and $c'$, and $\mathbf{x}$, $\mathbf{y}_c$ and $\mathbf{y}_{c'}$ are rewritten as $\mathbf{x}^0$, $\mathbf{y}^0_c$ and $\mathbf{y}^0_{c'}$, respectively, for notational consistency.

Next, we prove that the linear equations in $\Omega_{d+1}^{c,c'}$ are linearly independent.

Denote by $\mathbf{x}^i_j\in\mathbb{R}$ the $j$-th entry of $\mathbf{x}^i$. We write the coefficient matrix of $\Omega_{d+1}^{c,c'}$ as a $(d+1)$-by-$(d+1)$ dimensional square matrix
\begin{equation}\nonumber
A=
	\begin{bmatrix}
		1 & \mathbf{x}^0_1 & \mathbf{x}^0_2 & \dots & \mathbf{x}^0_d \\
		1 & \mathbf{x}^1_1 & \mathbf{x}^1_2 & \dots & \mathbf{x}^1_d \\
		\vdots  & \vdots  & \vdots & \vdots \\
		1 & \mathbf{x}^d_1 & \mathbf{x}^d_2 & \dots & \mathbf{x}^d_d \\
	\end{bmatrix},
\end{equation}
where the first column stores the coefficients for variable $B_{c,c'}$.
We prove that the linear equations in $\Omega_{d+1}^{c,c'}$ are linearly independent by showing that $A$ is a full rank matrix with probability $1$.
\begin{lemma}
\label{lma:fullrank}
When the perturbed instances are independently and uniformly sampled from a hypercube, the coefficient matrix $A$ of $\Omega_{d+1}^{c, c'}$ is a full rank matrix with probability $1$.
\end{lemma}
\begin{proof}

Denote by $A_i\in\mathbb{R}^{d+1}, i\in\{0, \ldots, d\}$, the $i$-th row of $A$, by $\overline{A_i}$ the sub-vector containing the last $d$ entries of $A_i$, that is, $\overline{A_i} = [x_1^i, x_2^i, \ldots, x_d^i]= x^i$.  Next, we prove the lemma by contradiction.

Assume the rank of matrix $A$ is not full. The last row of matrix $A$ must be a linear combination of the other rows. Denote by $\alpha_0, \dots, \alpha_{d-1}$ the weights of a linear combination, we write $A_d = \alpha_0 * A_0 + \alpha_1 * A_1 + \dots + \alpha_{d-1} * A_{d-1}$.

Since the first entry of every row vector $A_i$ is $1$, $\alpha_0 + \alpha_1 + \dots + \alpha_{d-1} = 1$.
Recall that $\overline{A_i}$ is a subvector of $A_i$ for all $i\in\{0, \ldots, d\}$, we have
\begin{equation}
    \label{eq:adbar}
    \overline{A_d} = \alpha_0 * \overline{A_0} + \alpha_1 * \overline{A_1} + \dots + \alpha_{d-1} * \overline{A_{d-1}}
\end{equation}

By plugging $\alpha_{d-1} = 1 - (\alpha_0 + \alpha_1 + \dots + \alpha_{d-2})$ into Equation \ref{eq:adbar}, we can derive
\begin{equation}
    \label{eq:ad}
    \overline{A_d} = \alpha_0 * (\overline{A_0} - \overline{A_{d-1}}) + \dots + \alpha_{d-2} * (\overline{A_{d-2}} - \overline{A_{d-1}}) + \overline{A_{d-1}}
\end{equation}

\nop{
Recall that $x_i$ is independently and uniformly sampled from a $d$-dimensional subspace, since $\overline{A_i} = x_i, i\in\{0, \ldots, d\}$, $\overline{A_i} - \overline{A_{d-1}}$ is also independently and uniformly sampled from a $d$-dimensional subspace.
}

Since Equation \ref{eq:ad} only contains the $d-1$ free variables $\alpha_0, \ldots, \alpha_{d-2}$, $\overline{A_d}$ is contained in the $(d-1)$-dimensional subspace $\mathcal{V}$ spanned by 
$(\overline{A_0} - \overline{A_{d-1}})$, $\dots$, $(\overline{A_{d-2}} - \overline{A_{d-1}})$.

Since $\overline{A_d}=x^d$, $\overline{A_d}$ is independently and uniformly sampled from a $d$-dimensional continuous space, and the probability that $\overline{A_d}$ is sampled from the $(d-1)$-dimensional subspace $\mathcal{V}$ is $0$. Therefore, the probability that Equation~\ref{eq:ad} holds is 0, which means $A_d$ cannot be represented as a linear combination of the other rows. This contradicts with the assumption that $A$ is not a full rank matrix. In sum, $A$ is a full rank matrix with probability $1$.
\end{proof}

Lemma~\ref{lma:fullrank} holds as long as the perturbed instances are independently and uniformly sampled from a hypercube. By Lemma~\ref{lma:fullrank}, $\Omega_{d+1}^{c,c'}$ is a determined linear equation system that is guaranteed to have a unique solution with probability $1$.

By solving each of the $C-1$ linear equation systems in $\{\Omega_{d+1}^{c,c'} \mid c'\in\{1, \ldots, C\}\setminus c\}$, we can easily determine the core parameters of $\mathbf{x}^0$ for each pair of $c$ and $c'$. 
Then, we can apply Equation~\ref{eqn:dc} to compute $D_c$.


The naive method introduced above is applicable when all sampled instances and the instance $\mathbf{x}^0$ have the same core parameters. 
However, since we do not know the model parameters of the PLM, 
we cannot guarantee that those instances have the same core parameters.
In other words, \nop{the ideal case may not hold in practice} the ideal case may not always hold in practice.  In sequel, the naive method cannot accurately compute $D_c$ all the time.
Indeed, when the ideal case assumption does not hold, the performance of the naive method can be arbitrarily bad.

\begin{theorem}
\label{thm:impossible}
Denote by $\beta^*$ the solution of $\Omega_{d+1}^{c,c'}$.
When the ideal case does not hold, the probability that $\beta^*$ is the core parameters of $\mathbf{x}^0$ is $0$ for at least one pair of classes $c$ and $c'$.
\end{theorem}


\begin{proof}
Denote by $\beta^i = (D_{c,c'}^i, B_{c,c'}^i), i\in\{0, \ldots, d\}$, the core parameters of $\mathbf{x}^i$, and by $\mathbb{P}(\beta^*=\beta^0)$ the probability of $\beta^* = \beta^0$.
We only need to show $\mathbb{P}(\beta^*=\beta^0)=0$ for at least one pair of $c$ and $c'$.

When the ideal case does not hold,
there is at least one sampled instance, denoted by $\mathbf{x}^i, i\in\{1, \ldots, d\}$, that does not have the same core parameters as $\mathbf{x}^0$. 
Therefore, $\beta^i\neq\beta^0$ for at least one pair of classes $c$ and $c'$.

By the definition of $\beta^i$, $\mathbf{x}^i$ satisfies 
\begin{equation}\nonumber
	{D_{c,c'}^i}^\top \mathbf{x}^i + B_{c,c'}^i = \ln(\frac{\mathbf{y}^i_c}{\mathbf{y}^i_{c'}}).
\end{equation}
If $\beta^*=\beta^0$, then $\mathbf{x}^i$ satisfies
\begin{equation}\nonumber
	{D_{c,c'}^0}^\top \mathbf{x}^i + B_{c,c'}^0 = \ln(\frac{\mathbf{y}^i_c}{\mathbf{y}^i_{c'}}).
\end{equation}
Therefore, a necessary condition for $\beta^*=\beta^0$ is that $\mathbf{x}^i$ satisfies
\begin{equation}
\label{eqn:out3}
	(D_{c,c'}^i - D_{c,c'}^0)^\top \mathbf{x}^i + (B_{c,c'}^i - B_{c,c'}^0) = 0.
\end{equation}
As a result, $\mathbb{P}(\beta^*=\beta^0)$ cannot be larger than the probability $P$ that $\mathbf{x}^i$ satisfies Equation~\ref{eqn:out3}.

Recall that $\beta^i\neq \beta^0$ for at least one pair of $c$ and $c'$.
The value of $P$ must fall into one of the following two cases.

Case 1: if $D_{c,c'}^i = D_{c,c'}^0$, then $B_{c,c'}^i\neq B_{c,c'}^0$.
In this case, no $\mathbf{x}^i$ satisfies Equation~\ref{eqn:out3}, thus $P=0$.

Case 2: if $D_{c,c'}^i \neq D_{c,c'}^0$, then $P$ is the probability that $\mathbf{x}^i$ is located on the hyperplane defined by Equation~\ref{eqn:out3}.
In this case, $P$ is still 0 because $\mathbf{x}^i$ is independently uniformly sampled from a $d$-dimensional hypercube.

In summary, $\mathbb{P}(\beta^*=\beta^0)\leq P=0$. The theorem follows.
\end{proof}


\nop{
In summary, 

the output of the naive method is correct only when the ideal assumption holds. Oth
}

In summary, the naive method only works in the idea case where all perturbed instances have the same core parameters as the input instance $\mathbf{x}^0$. The extremely strong assumption limits the method to be usable in practice. First, as discussed in Section~\ref{sec:rw}, it is impossible for users to heuristically select a perturbation distance that works for all instances. Second, if the perturbed instances have different core parameters, according to Lemma~\ref{lma:fullrank} and Theorem~\ref{thm:impossible}, the naive method may not obtain a correct interpretation. \nop{Third, as we cannot access the internal information of the PLMs, there is no method to assess the correctness of the results. In summary, the naive method is not usable in practice.}  Next, we develop \texttt{OpenAPI} to overcome these weaknesses.

\subsection{The \texttt{OpenAPI} Method}
\label{sec:openapi}
Now we are ready to introduce the \texttt{OpenAPI} method to reliably and accurately compute $D_c$. Different from the naive method, \texttt{OpenAPI} adaptively shrinks the perturbation distance until it computes the exact interpretations with probability $1$.

\nop{
making sure that the ideal assumption holds when \texttt{OpenAPI} produces an output.
}

For any two classes $c$ and $c'$, 
\texttt{OpenAPI} computes the core parameters $\beta^0$ of $\mathbf{x}^0$ by solving an \textbf{overdetermined} linear equation system with $d+2$ linear equations.
Denote by $\Omega_{d+2}^{c,c'}$ the overdetermined linear equation system.
We build the first $d+1$ linear equations of $\Omega_{d+2}^{c,c'}$ in the same way as the naive method. 
The $(d+2)$-th linear equation of $\Omega_{d+2}^{c,c'}$ is built by sampling an extra instance $\mathbf{x}^{d+1}$ in the neighborhood of the input instance $\mathbf{x}^0$.

Denote by $\beta^i=(D_{c,c'}^i, B_{c,c'}^i)$, $i\in\{0, \ldots, d+1\}$, the core parameters of $\mathbf{x}^i$ in classifying classes $c$ and $c'$.  We now show that, when $\Omega_{d+2}^{c,c'}$ has at least one solution, the solution is unique and is equal to every $\beta^i$ with probability $1$.

\begin{theorem}
\label{thm:main}
For any two classes $c$ and $c'$, if $\Omega_{d+2}^{c,c'}$ has at least one solution, then the solution is unique and is exactly $\beta^i$ for any $i\in\{0, \ldots, d+1\}$ with probability $1$.
\nop{
The proof of Theorem~\ref{thm:main} is given in Appendix~\ref{apd:3}.
}

\begin{proof}
Denote by $\Theta_i^{c,c'}$ the linear equation system that is constructed by removing the linear equation with respect to $\mathbf{x}^i$ from $\Omega_{d+2}^{c,c'}$ and keeping the rest $d+1$ linear equations.
Obviously, any solution of $\Omega_{d+2}^{c,c'}$ is a solution of $\Theta_i^{c,c'}$.

According to Lemma~\ref{lma:fullrank}, the coefficient matrix of $\Theta_i^{c,c'}$ is a full rank square matrix with probability $1$. 
This means that the probability that $\Theta_i^{c,c'}$ has a unique solution is $1$. 

Since $\Omega_{d+2}^{c,c'}$ has at least one solution and any solution of $\Omega_{d+2}^{c,c'}$ is a solution of $\Theta_i^{c,c'}$, the solution of $\Omega_{d+2}^{c,c'}$ is unique, and is equal to the solution of $\Theta_i^{c,c'}$.
Next, we prove that, with probability $1$, the solution of $\Omega_{d+2}^{c,c'}$ is exactly $\beta^i$ for any $i\in\{0, \ldots, d+1\}$.

Denote by $\beta^* = (D_{c,c'}^*, B_{c,c'}^*)$ the unique solution of $\Omega_{d+2}^{c,c'}$, and by $\mathbb{P}(\beta^*=\beta^i)$ the probability of $\beta^*=\beta^i$. 
We only need to show $\mathbb{P}(\beta^*\neq\beta^i)=0$ for any $i\in\{0, \ldots, d+1\}$.

By the definition of $\beta^i$, $\mathbf{x}^i$ satisfies 
\begin{equation}\nonumber
	{D_{c,c'}^i}^\top \mathbf{x}^i + B_{c,c'}^i = \ln(\frac{\mathbf{y}^i_c}{\mathbf{y}^i_{c'}}).
\end{equation}
Since $\beta^*$ is the unique solution of $\Omega_{d+2}^{c,c'}$, $\mathbf{x}^i$ also satisfies 
\begin{equation}\nonumber
	{D_{c,c'}^*}^\top \mathbf{x}^i + B_{c,c'}^*  = \ln(\frac{\mathbf{y}^i_c}{\mathbf{y}^i_{c'}}).
\end{equation}
Therefore, $\mathbf{x}^i$ must satisfy 
\begin{equation}
\label{eqn:out4}
	(D_{c,c'}^* - D_{c,c'}^i)^\top \mathbf{x}^i + (B_{c,c'}^* - B_{c,c'}^i) = 0
\end{equation}

Consequently, a necessary condition for $\beta^*\neq \beta^i$ is that $\mathbf{x}^i$ is located on the hyperplane $\mathcal{H}$ defined by Equation~\ref{eqn:out4}.
Therefore, $\mathbb{P}(\beta^*\neq\beta^i)$ cannot be larger than the probability $P$ that $\mathbf{x}^i$ is located on $\mathcal{H}$.
The value of $\mathbb{P}(\beta^*\neq\beta^i)$ must fall into one of the following three cases.

Case 1: if $D_{c,c'}^* = D_{c,c'}^i$ and $B_{c,c'}^* = B_{c,c'}^i$, then $\beta^*=\beta^i$, which means $\mathbb{P}(\beta^*\neq\beta^i)=0$.

Case 2: if $D_{c,c'}^* = D_{c,c'}^i$ and $B_{c,c'}^* \neq B_{c,c'}^i$, then no $\mathbf{x}^i$ satisfies Equation~\ref{eqn:out4}, which means $P=0$. 
Thus, $\mathbb{P}(\beta^*\neq\beta^i)\leq P=0$.

Case 3: if $D_{c,c'}^* \neq D_{c,c'}^i$, 
because $\mathbf{x}^0$ is drawn from an underlying continuous distribution in the $d$-dimensional space $\mathcal{X}$~\cite{nasrabadi2007pattern}, and each $\mathbf{x}^i, i\in\{1, \ldots, d+1\}$, is uniformly sampled from a $d$-dimensional hypercube, we have $P=0$.
Therefore, $\mathbb{P}(\beta^*\neq\beta^i)\leq P=0$.

In summary, $\mathbb{P}(\beta^*\neq\beta^i)=0$ and the theorem follows.
\end{proof}
\end{theorem}

\begin{algorithm}[t]
\SetAlgoLined
\setcounter{AlgoLine}{0}
\KwIn{$\mathcal{A}\coloneqq$ the API of a PLM, $c\coloneqq$ the class $c$ to interpret, $\mathbf{x}^0\coloneqq$ the instance to interpret, $m\coloneqq$the maximum number of iterations.}
\KwOut{$D^*_c:=$ the value of $D_c$ computed by \texttt{OpenAPI}, $r:=$ the hypercube edge length.}

Initialize: $r\leftarrow 1.0$, $\mathcal{I}\leftarrow\emptyset$, $D^*_{c} \leftarrow null$.

\While{$m\neq 0$}{
    Sample $d+1$ points in the hypercube with edge length $r$: $S\leftarrow\{\mathbf{x}^1, \ldots, \mathbf{x}^{d+1} \}$.
    
    \For{each $c' \in \{1, \ldots, C\}\setminus c$}{
	Construct $\Omega_{d+2}^{c,c'}$ by $d+2$ points in $S\cup\mathbf{x}^0$.
	
    	\textbf{If} $\Omega_{d+2}^{c,c'}$ has a solution $\beta^*$ \textbf{then} 
		$\mathcal{I}\leftarrow \mathcal{I}\cup\beta^*$.
    }
    
    \eIf{$|\mathcal{I}| < C-1$}{
    	$\mathcal{I}\leftarrow\emptyset$, $r\leftarrow r/2$.
    }{
    	Compute $D^*_c$ from $\mathcal{I}$ by Equation~\ref{eqn:dc}, and break.
    }
    $m\leftarrow m - 1$
}

\Return $D^*_c$, $r$.
\caption{\texttt{OpenAPI}($\mathcal{A}, c, \mathbf{x}^0, m$)}
\label{algo:interpret}
\end{algorithm}

According to Theorem~\ref{thm:main}, if $\Omega_{d+2}^{c,c'}$ has a solution, 
then it is the core parameters $\beta^0$ of $\mathbf{x}^0$ with probability $1$. 
In this case, we can directly compute $\beta^0$ as the closed form solution to $\Omega_{d+2}^{c,c'}$.
If $\Omega_{d+2}^{c,c'}$ has no solution, we can reconstruct it by randomly sampling a new set of instances in the neighborhood of $\mathbf{x}^0$, and solve the corresponding linear equation system. 
This iteration of reconstructions continues until we sample a set of instances that have the same core parameters as $\mathbf{x}^0$. Then, we can find the solution to $\Omega_{d+2}^{c,c'}$, which is $\beta^0$ with probability $1$.

Recall that all instances within the same locally linear region have the same core parameters. If we sample instances from a \textbf{proper hypercube} that is contained in the locally linear region of $\mathbf{x}^0$, then the instances sampled certainly have the same core parameters as $\mathbf{x}^0$, and we are sure to find the valid solution $\beta^0$.

\nop{
An input instance $\mathbf{x}^0$ does not have a proper hypercube with $r>0$ if it is located on the boundary of a locally linear region. 
}

\nop{
With probability equal to 0, a $\mathbf$

a proper hypercube does not exist when $\mathbf{x}^0$ locates on the boundary of a locally linear region.

When $\mathbf{x}^0$ locates on the boundary of a locally linear region, it is impossible to 
}

Intuitively, a hypercube with smaller edge length $r$ is more likely to be contained by the locally linear region of $\mathbf{x}^0$. 
However, it is impractical to empirically set one value of $r$ to fit all PLMs and arbitrary instances to be interpreted, because the sizes of locally linear regions vary significantly for different PLMs, and the maximum $r$ of a proper hypercube can be arbitrarily small for an input instance that is very close to the boundary of a locally linear region.
Therefore, as described in Algorithm~\ref{algo:interpret}, \texttt{OpenAPI} adaptively finds a proper hypercube by reducing the edge length $r$ by half in each iteration of reconstruction.

\nop{
The only special case that \texttt{OpenAPI} may not work well is when $\mathbf{x}^0$ locates on the boundary of a locally linear region.
In this case, no hypercube with $r>0$ can be contained by the locally linear region. 
However, 
}

As long as $\mathbf{x}^0$ is contained in a locally linear region,
\texttt{OpenAPI} eventually can find a proper hypercube and compute a valid output, denoted by $D^*_c$.
If $\mathbf{x}^0$ is located on the boundary of a locally linear region, then there is no proper hypercube with $r>0$ for $\mathbf{x}^0$, and \texttt{OpenAPI} may fail to return a valid output.
However, since the probability that $\mathbf{x}^0$ is located on the boundary of a locally linear region is 0, the probability that \texttt{OpenAPI} returns the valid $D^*_c$ is still $1$.  To guarantee \texttt{OpenAPI} terminates even in the unlikely case that $\mathbf{x}^0$ is located on the boundary of a locally linear region, \texttt{OpenAPI} stops after a certain number of iterations, which is a system parameter. In our experiments, we set the maximum number of iterations for \texttt{OpenAPI} as  100. However, since the probability that $\mathbf{x}^0$ is located on the boundary of a locally linear region is 0, the  non-terminating case never happened in our experiments, and \texttt{OpenAPI} always terminates in less than 20 iterations. If \texttt{OpenAPI} cannot find a proper hypercube within the maximum number of iterations, the smallest edge length $r$, which is constructed at the last iteration, will be returned.

Since \texttt{OpenAPI} adaptively finds a proper hypercube, the initial value of $r$ has little influence on the accuracy of \texttt{OpenAPI}. 
Thus, we simply initialize it as $r=1.0$ in our experiments.

\texttt{OpenAPI} has three major advantages. 
First, \texttt{OpenAPI} computes interpretations in closed form, and provides a solid theoretical guarantee on the exactness of interpretations.
Second, our interpretation is consistent for all instances in the same locally linear region. This is because all instances contained in the same locally linear region have the same decision features, which are accurately identified by \texttt{OpenAPI}.
Last, \texttt{OpenAPI} is highly efficient, of time complexity $O( T \cdot C (d+2) ^ 3)$, where $d$ and $C$ are constants for a PLM, and $T$ is the number of iterations of reconstruction.

\section{Experiments}
\label{sec:exp}

In this section, we evaluate the performance of \texttt{OpenAPI} by investigating the following four questions: (1) Can \texttt{OpenAPI} effectively explain model predictions? (2) Are the interpretations consistent? (3) How well are the perturbed instances being used for interpretations? (4) Are the computed interpretations exact?

To demonstrate that \texttt{OpenAPI} can effectively interpret the predictions of PLMs, we compare \texttt{OpenAPI} with four baseline interpretation methods, \texttt{Saliency Maps}~\cite{simonyan2013deep}, \texttt{Gradient * Input}~\cite{shrikumar2016not},  \texttt{Integrated Gradient}~\cite{sundararajan2017axiomatic}, and \texttt{LIME}~\cite{ribeiro2016should}. The first three gradient-based methods~\cite{ancona2018towards} require to access the model parameters. \texttt{LIME} can interpret the predictions of PLMs with only API access. 

\texttt{Saliency Maps} interprets a prediction by taking the absolute value of the partial derivative of the prediction with respect to the input features.
\texttt{Gradient * Input} uses the feature-wise product between the partial derivative and the input itself as the interpretation for a prediction.
Rather than computing the partial derivative of the input instance $\mathbf{x}^0$, \texttt{Integrated Gradient} computes the average partial derivatives when the input varies along a linear path from a baseline point to $\mathbf{x}^0$.
\texttt{LIME} interprets the predictions of a classifier by training an interpretable model on the outputs of the classifier in a heuristically selected neighborhood of the input instance. We adopt the same experiment settings used in~\cite{shrikumar2017learning} and~\cite{chu2018exact} for \texttt{Integrated Gradient} and \texttt{LIME}, respectively.

\nop{\mc{In the last two subsections, we show that \texttt{OpenAPI} can exactly compute our proposed interpretation $D_c$ with only API access to the PLMs.}}

To evaluate the capability of interpretation with only API access to PLMs, in addition to the naive method discussed in Section~\ref{method:naive}, we design two more baselines by slightly extending  \texttt{ZOO}~\cite{chen2017zoo} and \texttt{LIME}~\cite{ribeiro2016should} as follows. 

\texttt{ZOO} is a zeroth-order approximation method approximating the gradients of functions. It first samples $d$ pairs of instances by perturbing $\mathbf{x}^0$ back-and-forth along each axis of $\mathbb{R}^d$ for a heuristically fixed \textbf{perturbation distance} $h$. Then, it estimates the gradient of a model with respect to $\mathbf{x}^0$ by computing the symmetric difference quotient~\cite{lax2014calculus} between each pair of sampled instances. Equation~\ref{eqn:lineareq} clearly shows that the derivative of $\ln(\frac{\mathbf{y}_c}{\mathbf{y}_{c'}})$ with respect to $\mathbf{x}$ is exactly $D_{c, c'}$. Thus, it is natural to use \texttt{ZOO} to estimate $D_{c, c'}$. Then $D_c$ is computed from the estimated $D_{c, c'}$ in the same way as Equation~\ref{eqn:dc}.

\texttt{LIME} interprets predictions in the one-vs-the-rest way~\cite{ribeiro2016should}. It is easy to extend \texttt{LIME} such that it uses $D_c$ as its interpretations. Rather than training a linear model to approximate the predicted probability $\mathbf{y}_c$ of a perturbed instance, the extended \texttt{LIME} tries to fit $\ln(\frac{\mathbf{y}_c}{\mathbf{y}_{c'}})$ of the perturbed instances. Because of the linear relationship between an instance $\mathbf{x}$ and the corresponding value $\ln(\frac{\mathbf{y}_c}{\mathbf{y}_{c'}})$, the coefficients of the linear model are an approximation to $D_{c, c'}$. Similarly to \texttt{ZOO}, $D_{c}$ is computed from the estimated $D_{c, c'}$.
In our experiments, two types of linear regression models are used as approximators. The one using regular linear regression is called \texttt{Linear Regression LIME} and the one using ridge regression is called \texttt{Ridge Regression LIME}.

\nop{All the three methods can interpret a PLM hidden behind an API, since they do not access training data or model parameters.
None of them, however, can effectively handle the classical problem of softmax saturation. When an input instance $\mathbf{x}^0$ is classified with a probability extremely close to $1$, the saturated softmax function loses a lot of gradient information. In such a case, \texttt{LIME} and \texttt{ZOO} have little gradient information to accurately approximate the $D_c$ of a PLM~\cite{chu2018exact}, and the linear equation system of \texttt{OpenAPI} is ill-conditioned.
To obtain informative comparison results, we evaluate the performance of all three methods on the instances with unsaturated softmax function, and ignore the instances that are classified with a probability larger than $0.9999$.}

\nop{
which makes it impossible for any method to recover $D_c$ from the API of a PLM.

Specifically, when the softmax function is saturated, 
\texttt{LIME} and \texttt{ZOO} will have little gradient information to accurately approximate the $D_c$ of a PLM~\cite{chu2018exact}, and the linear equation system of \texttt{OpenAPI} will be ill-conditioned.
As a result, we evaluate the performance of all methods on the instances with unsaturated softmax function, and ignore the instances that are classified with a probability larger than 0.9999.
}
\nop{
Evaluation in such a common bad case of all methods doesn't produce informative comparison results, 
}

We use the published Python codes of \texttt{Integrated Gradient}\footnote{\url{https://github.com/ankurtaly/Integrated-Gradients}}, \texttt{LIME}\footnote{\url{https://github.com/marcotcr/lime}} and \texttt{ZOO}\footnote{\url{https://github.com/huanzhang12/ZOO-Attack}}. The remaining algorithms are implemented using the PyTorch library~\cite{paszke2017automatic}.
All experiments are conducted on a server with two Xeon(R) Silver 4114 CPUs (2.20GHz), four Tesla P40 GPUs, 400GB main memory, and a 1.6TB SSD running Cenos 7 OS.
\nop{We will publish the code of \texttt{OpenAPI} as well as the data sets to public upon acceptance of our paper.} Our source code is published at GitHub $\langle$\url{https://github.com/researchcode2/OpenAPI}$\rangle$.

\begin{table}[t]
    \centering
    \begin{tabular}{|c|c|c|c|c|}
    \hline
    \multirow{2}{*}{Data Sets} & \multicolumn{2}{c|}{FMNIST} & \multicolumn{2}{c|}{MNIST} \\
    \cline{2-5}
    & Train & Test & Train & Test \\
    \hline
        PLNN & 0.888 & 0.865 & 0.980 & 0.971 \\
    \hline
        LMT & 0.950 & 0.870 & 0.991 & 0.949 \\
    \hline
    \end{tabular}
    \caption{The training and testing accuracies of all models}
    \label{tab:my_label}
\end{table}

We conduct all experiments across two public datasets, FMNIST \cite{xiao2017/online} and MNIST \cite{lecun-mnisthandwrittendigit-2010}. FMNIST contains fashion images in 10 categories and MNIST contains images of handwritten digits from 0 to 9. Both datasets consist of a training set of 60,000 examples and a test set of 10,000 examples.
\nop{We built two data sets, named FMNIST-1 and FMNIST-2, from the Fashion MNIST data set~\cite{xiao2017/online} in a similar way as Chu~\emph{et~al.}~\cite{chu2018exact}. FMNIST-1 contains the images of \emph{Coat} and \emph{Pullover}, and FMNIST-2 consists of the images of \emph{T-shirt} and \emph{Shirt}. }
We represent each of the 28-by-28 gray scale images by cascading the 784 pixel values into a 784-dimensional feature vector. The pixel values are normalized to the range $[0, 1]$.
\nop{Table~\ref{tab:data_desc} gives the details of FMNIST-1 and FMNIST-2.}

\nop{
The Fashion MNIST data set is available online\footnote{https://github.com/zalandoresearch/fashion-mnist}.
}

On each dataset, we train a Logistic Model Tree (LMT)~\cite{landwehr2005logistic} and a Piecewise Linear Neural Network (PLNN)~\cite{chu2018exact} as the target PLMs to interpret.
The classification performance of all models are shown in Table~\ref{tab:my_label}.

\nop{To train a LMT, we first use the standard C4.5 algorithm~\cite{quinlan2014c4} to train a decision tree, and then train a sparse logistic regression classifier}

Following the design in~\cite{landwehr2005logistic}, we use the standard C4.5 algorithm~\cite{quinlan2014c4} to select the pivot feature for each node and a sparse multinomial logistic regression classifier is trained on each leaf node of the tree. To prevent overfitting, we adopt two stopping criteria. A node is not further split if it contains less than 100 training instances or the accuracy of the regression classifier is greater than 99\%.
Since every leaf node of a LMT is a locally linear classifier, the leaf node itself corresponds to a locally linear region, and we can directly extract the ground truth decision features for an input instance $\mathbf{x}^0$ from the multinomial logistic regression classifier of the leaf node containing $\mathbf{x}^0$.

To train a PLNN, we use the standard back-propagation to train a fully-connected network that adopts the widely used activation function ReLU~\cite{glorot2011deep}.
The numbers of neurons from the input layer to the output layer are 784, 256, 128, 100 and 10, respectively. This network is used as a baseline model on the website $\langle$\url{https://github.com/zalandoresearch/fashion-mnist}$\rangle$ of FMNIST.
We use OpenBox~\cite{chu2018exact} to compute the locally linear regions and the ground truth decision features $D_c$ for an input instance of a PLNN.

Since \texttt{LIME} is too slow to process all instances in 24 hours, for each of FMNIST and MNIST, we uniformly sample 1000 instances from the testing set, and conduct all experiments for \nop{\texttt{LIME}}all methods on the sampled instances.

\subsection{Can \texttt{OpenAPI} effectively Interpret Model Predictions?}

Good interpretations should be easily understood by human being. In this subsection, we first conduct a case study to illustrate the effectiveness of the interpretations. Then, we quantitatively evaluate the effectiveness of the interpretations given by \texttt{OpenAPI} and the four baseline methods. The three gradient-based methods are allowed to use the parameter information of the PLMs to compute their interpretations. \texttt{LIME} and \texttt{OpenAPI} are only allowed to use the APIs of the PLMs.

\newcommand{\parawidthExpone}{15mm}
\newcommand{\colorbarw}{4.2mm}
\begin{figure}[t]
\centering
\subfloat[Boot]{
\includegraphics[page=1, width=\parawidthExpone]{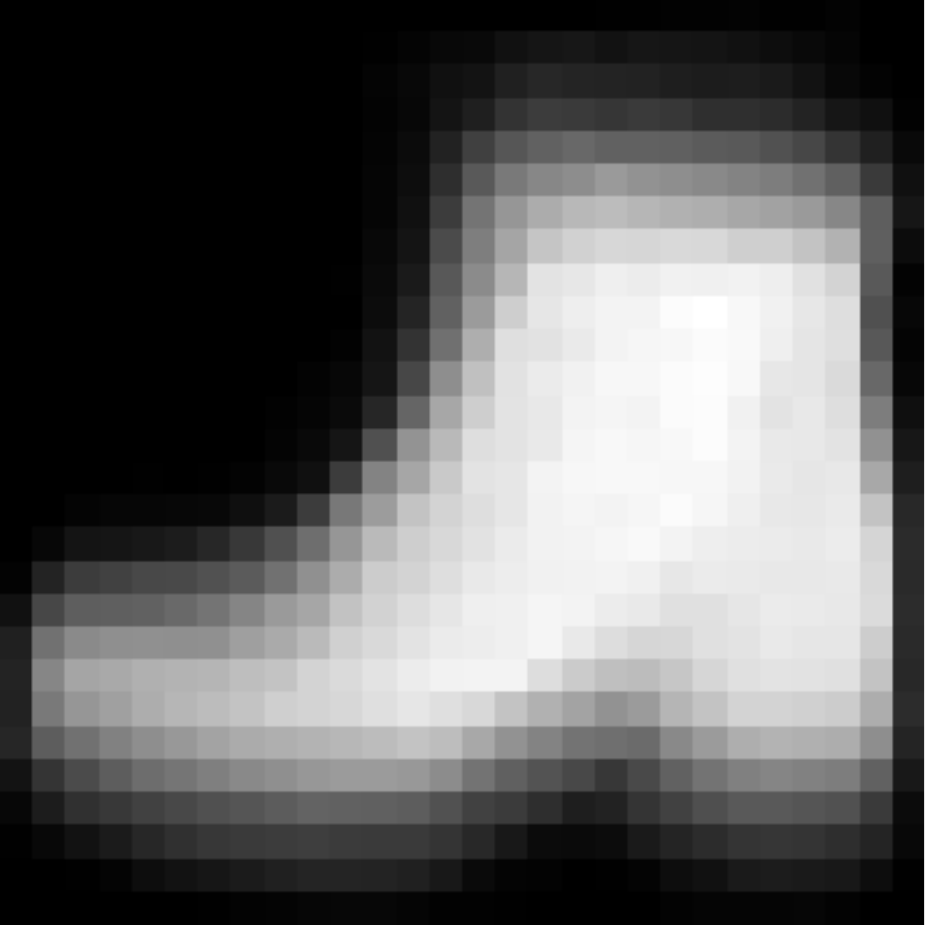}
}
\subfloat[Pullover]{
\includegraphics[page=3, width=\parawidthExpone]{Figures/DecisionFeature/MLP_FMNIST_GroundTruth_100_FMNIST.pdf}
}
\subfloat[Coat]{
\includegraphics[page=9, width=\parawidthExpone]{Figures/DecisionFeature/MLP_FMNIST_GroundTruth_100_FMNIST.pdf}
}
\subfloat[Sneaker]{
\includegraphics[page=13, width=\parawidthExpone]{Figures/DecisionFeature/MLP_FMNIST_GroundTruth_100_FMNIST.pdf}
}
\subfloat[T-shirt]{
\includegraphics[page=19, width=\parawidthExpone]{Figures/DecisionFeature/MLP_FMNIST_GroundTruth_100_FMNIST.pdf}
}

\subfloat[P, Boot]{
\includegraphics[page=2, width=\parawidthExpone]{Figures/DecisionFeature/MLP_FMNIST_GroundTruth_100_FMNIST.pdf}
}
\subfloat[P, Pull.]{
\includegraphics[page=4, width=\parawidthExpone]{Figures/DecisionFeature/MLP_FMNIST_GroundTruth_100_FMNIST.pdf}
}
\subfloat[P, Coat]{
\includegraphics[page=10, width=\parawidthExpone]{Figures/DecisionFeature/MLP_FMNIST_GroundTruth_100_FMNIST.pdf}
}
\subfloat[P, Sneak.]{
\includegraphics[page=14, width=\parawidthExpone]{Figures/DecisionFeature/MLP_FMNIST_GroundTruth_100_FMNIST.pdf}
}
\subfloat[P, T-shirt]{
\includegraphics[page=20, width=\parawidthExpone]{Figures/DecisionFeature/MLP_FMNIST_GroundTruth_100_FMNIST.pdf}
}

\subfloat[L, Boot]{
\includegraphics[page=2, width=\parawidthExpone]{Figures/DecisionFeature/LMT_FMNIST_GroundTruth_100_FMNIST.pdf}
}
\subfloat[L, Pull.]{
\includegraphics[page=4, width=\parawidthExpone]{Figures/DecisionFeature/LMT_FMNIST_GroundTruth_100_FMNIST.pdf}
}
\subfloat[L, Coat]{
\includegraphics[page=10, width=\parawidthExpone]{Figures/DecisionFeature/LMT_FMNIST_GroundTruth_100_FMNIST.pdf}
}
\subfloat[L, Sneak.]{
\includegraphics[page=14, width=\parawidthExpone]{Figures/DecisionFeature/LMT_FMNIST_GroundTruth_100_FMNIST.pdf}
}
\subfloat[L, T-shirt]{
\includegraphics[page=20, width=\parawidthExpone]{Figures/DecisionFeature/LMT_FMNIST_GroundTruth_100_FMNIST.pdf}
}

\caption{The averaged images of the selected FMNIST classes and their averaged decision features of PLNN (P) and LMT (L) computed by \texttt{OpenAPI}. ``Pull." and ``Sneak." are short for pullover and sneaker, respectively}
\label{fig:dc}
\end{figure}

Following the tradition of interpretation visualization~\cite{ancona2018towards}, we show the decision features as heatmaps, where red and blue colors indicate respectively features that contribute positively to the activation of the target output and features having a suppressing effect. The first row of Figure~\ref{fig:dc} shows the averaged images of five selected classes from FMNIST. For each  class, its averaged decision features of the trained PLNN and LMT are shown in the second and third rows, respectively.

Comparing the heatmaps with their corresponding averaged original images, it is clear that the decision features legibly highlight the image parts with strong semantical meanings, like the heal of boots, the shoulder of pullovers, the collar of coats, the surface of sneakers, and the short sleeves of T-shirts. A closer look at the averaged images suggests that the highlighted parts describe the differences between one type of objects against the others. 

Since the LMT is trained with sparse constraints, the decision features of the LMT are sparser than the ones of the PLNN. As a result, the PLNN captures more details of the objects.  Since both the LMT and the PLNN are trained on the same training data, 
the decision features learnt by the LMT highlight similar image patterns as the decision features of the PLNN.
This demonstrates the robustness of our proposed decision features in accurately interpreting general PLMs.

To quantitatively evaluate the effectiveness of interpretations, we adopt the evaluation method used by Ancona \emph{et~al.}~\cite{ancona2018towards}. The method assumes that a good interpretation model should identify features that are more relevant to the predictions. Therefore, modifications on those relevant features should result in sensible variations on the predictions. Following this idea, we modify the input features according to their weights in the computed interpretations as follows.

For each interpretation method, given an input instance $\mathbf{x}^0$ with predicted label $c$, we sort the input features in the descending order of their absolute weights. Based on the ranking, we proceed iteratively altering the input features one at a time and up to 200 features. As the features having positive (negative) weights support (opposite) to predict $\mathbf{x}^0$ as $c$, to decrease the confidence of a PLM on class $c$, we replace the input features of positive and negative weights by 0 and 1, respectively.\nop{we depreciate the support effects and enhance the suppressing effects of the features by replacing the input features with positive (negative) weights by 0 (1).} The changes on the predictions are evaluated by two metrics, the \textbf{change of prediction probability (CPP)} and the \textbf{number of label-changed instance (NLCI)}~\cite{chu2018exact}. \textbf{CPP} is the absolute change of the probability of classifying $\mathbf{x}^0$ as $c$ and \textbf{NLCI} is the number of instances whose predicted labels change after their features being altered.

As shown in Figure~\ref{fig:effect}, \texttt{Saliency Maps} performs worst among all methods. The result is consistent with the conclusion in \cite{ancona2018towards} that the instances may have features that opposite the predictions of some classes. Those features play an important role in interpreting the model predictions and can only be detected by signed interpretation methods.
As shown in Figure~\ref{fig:effect} and mentioned by Ancona \emph{et~al.}~\cite{ancona2018towards}, \texttt{Gradient * Input} captures important features better than \texttt{Integrated Gradient}. The latter involves the gradients of the unrelated instances into interpretations, therefore cannot precisely interpret the predictions.
As expected, \texttt{LIME} performs poorer than most of the gradient-based methods due to the fact that \texttt{LIME} has no access to the model parameters. The lack of internal information prevents it from getting accurate interpretations.
However, only with API access to the PLMs, \texttt{OpenAPI} outperforms the other methods most of the time, because our method computes the decision features that are exactly used by the PLMs in prediction. The good performance demonstrates the effectiveness of our method. 


\begin{figure}[h]
\centering
\subfloat[FMNIST (LMT)]{
\includegraphics[page=1, width=\figwidthone\linewidth]{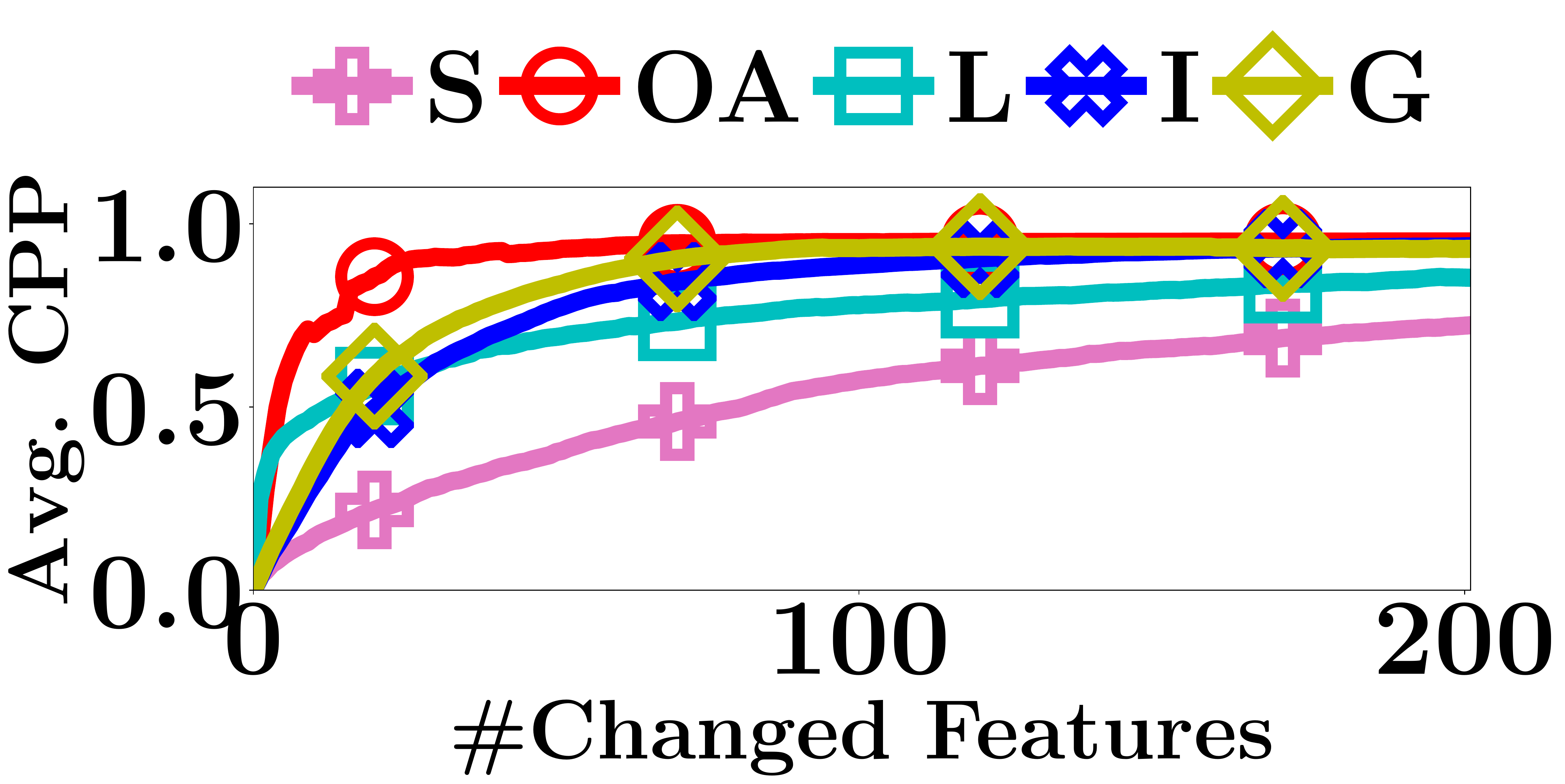}
}
\subfloat[FMNIST (PLNN)]{
\includegraphics[page=1, width=\figwidthone\linewidth]{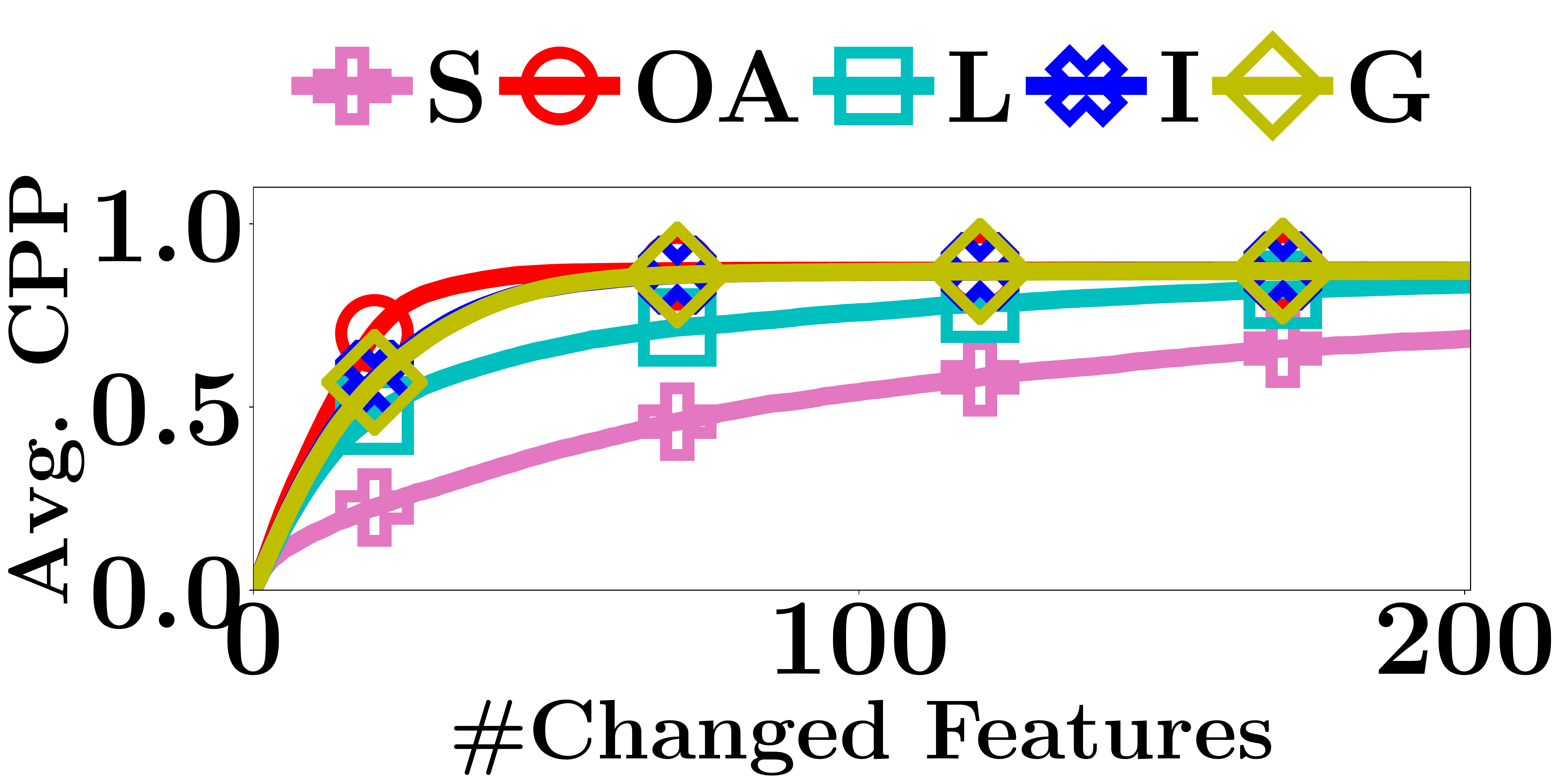}
}
\qquad
\subfloat[MNIST (LMT)]{
\includegraphics[page=1, width=\figwidthone\linewidth]{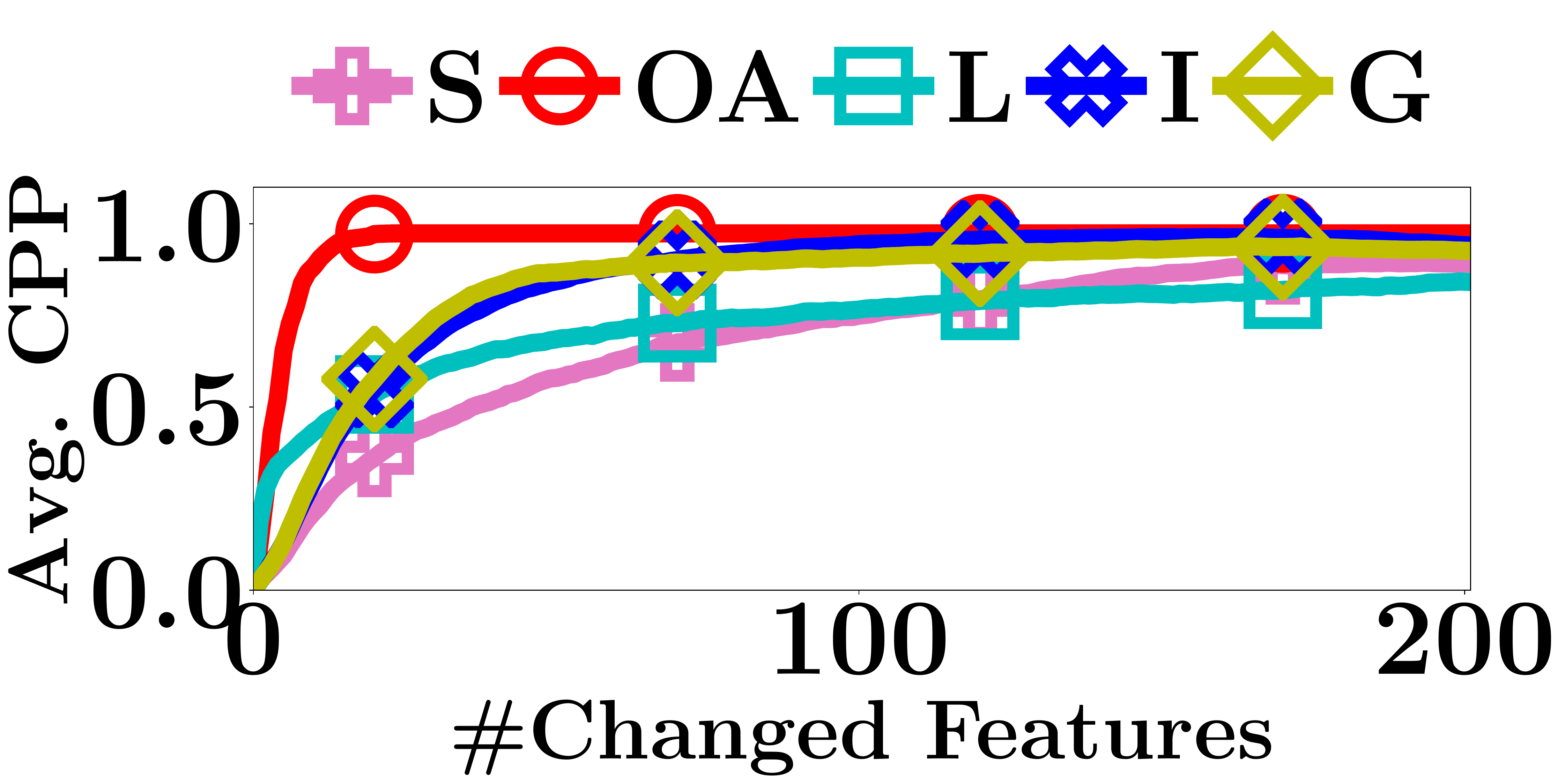}
}
\subfloat[MNIST (PLNN)]{
\includegraphics[page=1, width=\figwidthone\linewidth]{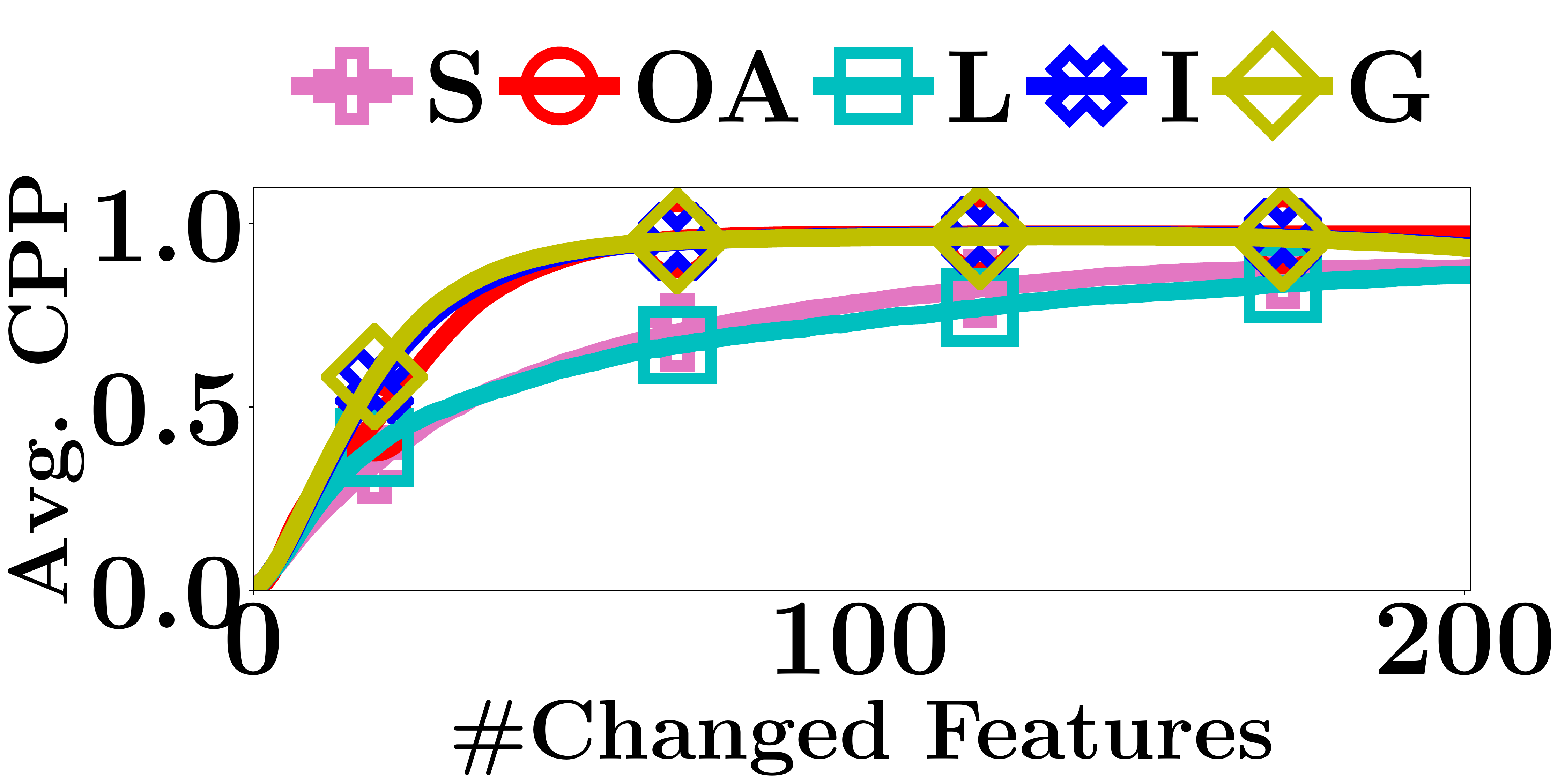}
}
\qquad
\subfloat[FMNIST (LMT)]{
\includegraphics[page=2, width=\figwidthone\linewidth]{Figures/DecisionFeature/LMT_FMNIST_100_EXP6.pdf}
}
\subfloat[FMNIST (PLNN)]{
\includegraphics[page=2, width=\figwidthone\linewidth]{Figures/DecisionFeature/MLP_FMNIST_100_EXP6.pdf}
}
\qquad
\subfloat[MNIST (LMT)]{
\includegraphics[page=2, width=\figwidthone\linewidth]{Figures/DecisionFeature/LMT_MNIST_100_EXP6.pdf}
}
\subfloat[MNIST (PLNN)]{
\includegraphics[page=2, width=\figwidthone\linewidth]{Figures/DecisionFeature/MLP_MNIST_100_EXP6.pdf}
}

\centering
\caption{The effectiveness of different interpretation methods}\nop{the interpretations computed by different methods. (a)-(d) and (e)-(h) show the average CPP and average NLCI of all methods, respectively. ``S", ``OA", ``I", ``G", and ``L" are short for \texttt{Saliency Maps}, \texttt{OpenAPI}, \texttt{Integrated Gradient}, \texttt{Gradient * Input}, and \texttt{LIME} respectively. Each curve represents a method, and is plotted using 200 data points.  We use different markers to make the curves more legible. }

\label{fig:effect}
\end{figure}

\begin{figure}[t]
\centering
\subfloat[FMNIST (LMT)]{
\includegraphics[page=1, width=\figwidthone\linewidth]{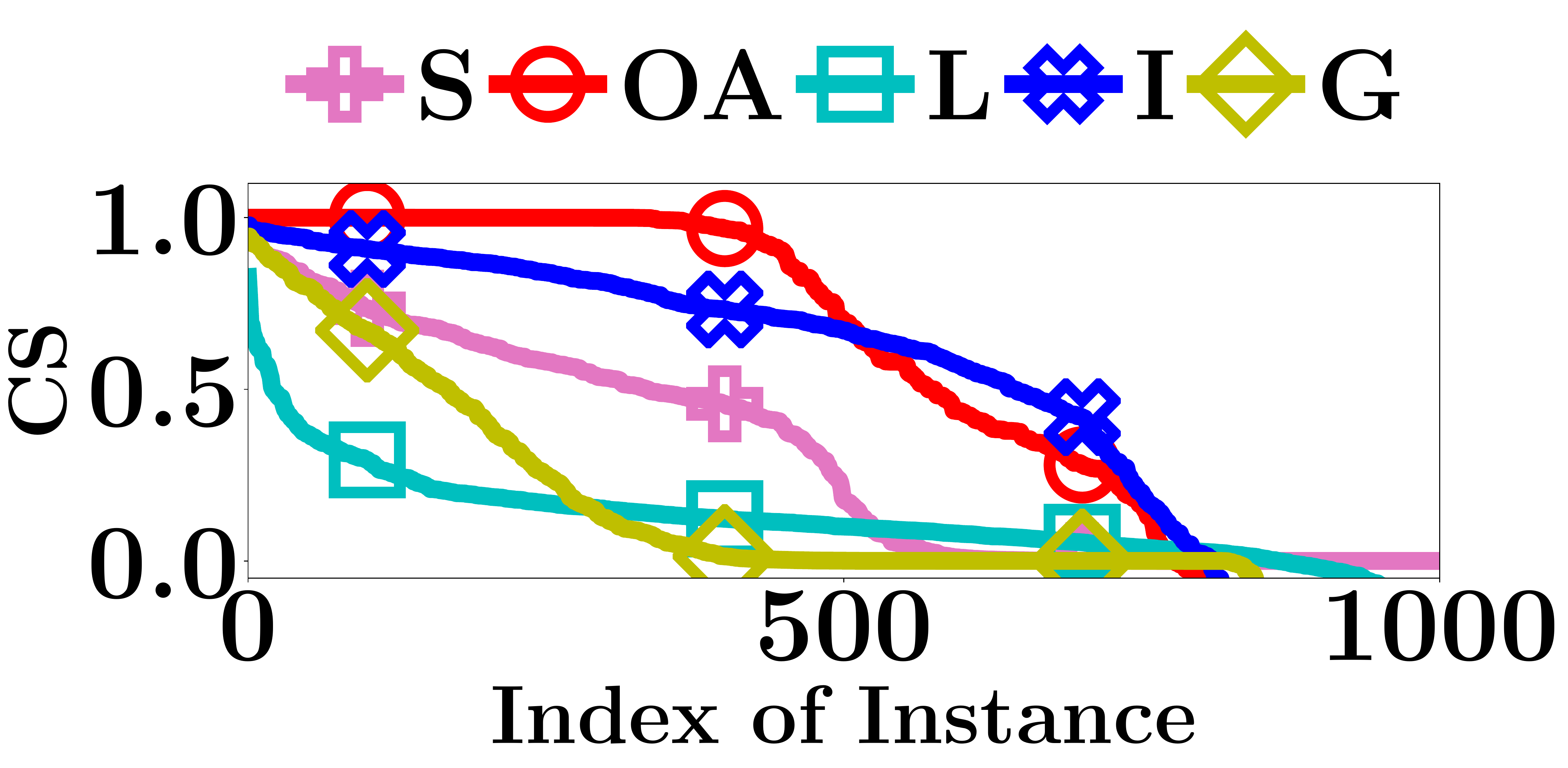}
}
\subfloat[FMNIST (PLNN)]{
\includegraphics[page=1, width=\figwidthone\linewidth]{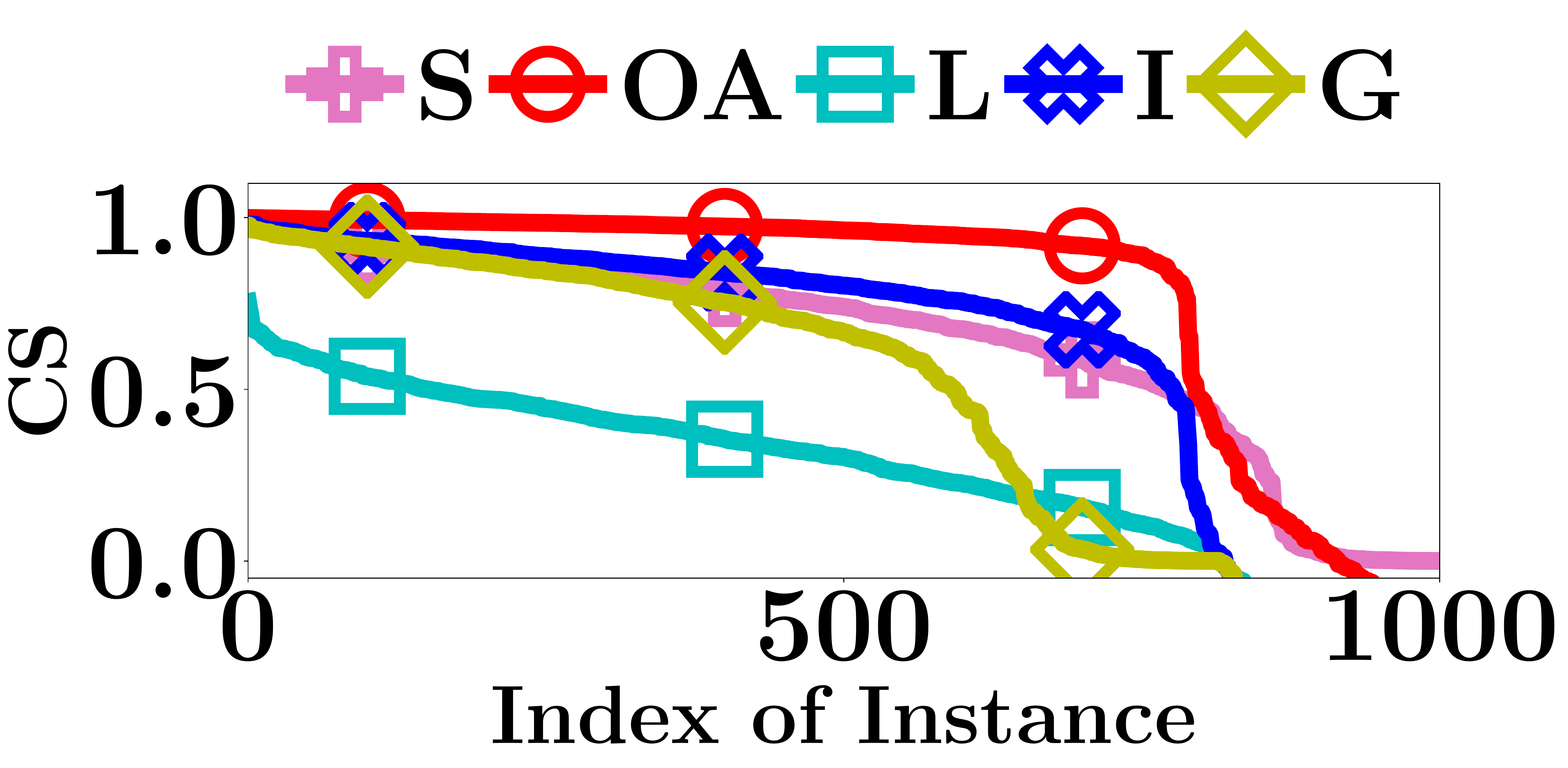}
}
\qquad
\subfloat[MNIST (LMT)]{
\includegraphics[page=1, width=\figwidthone\linewidth]{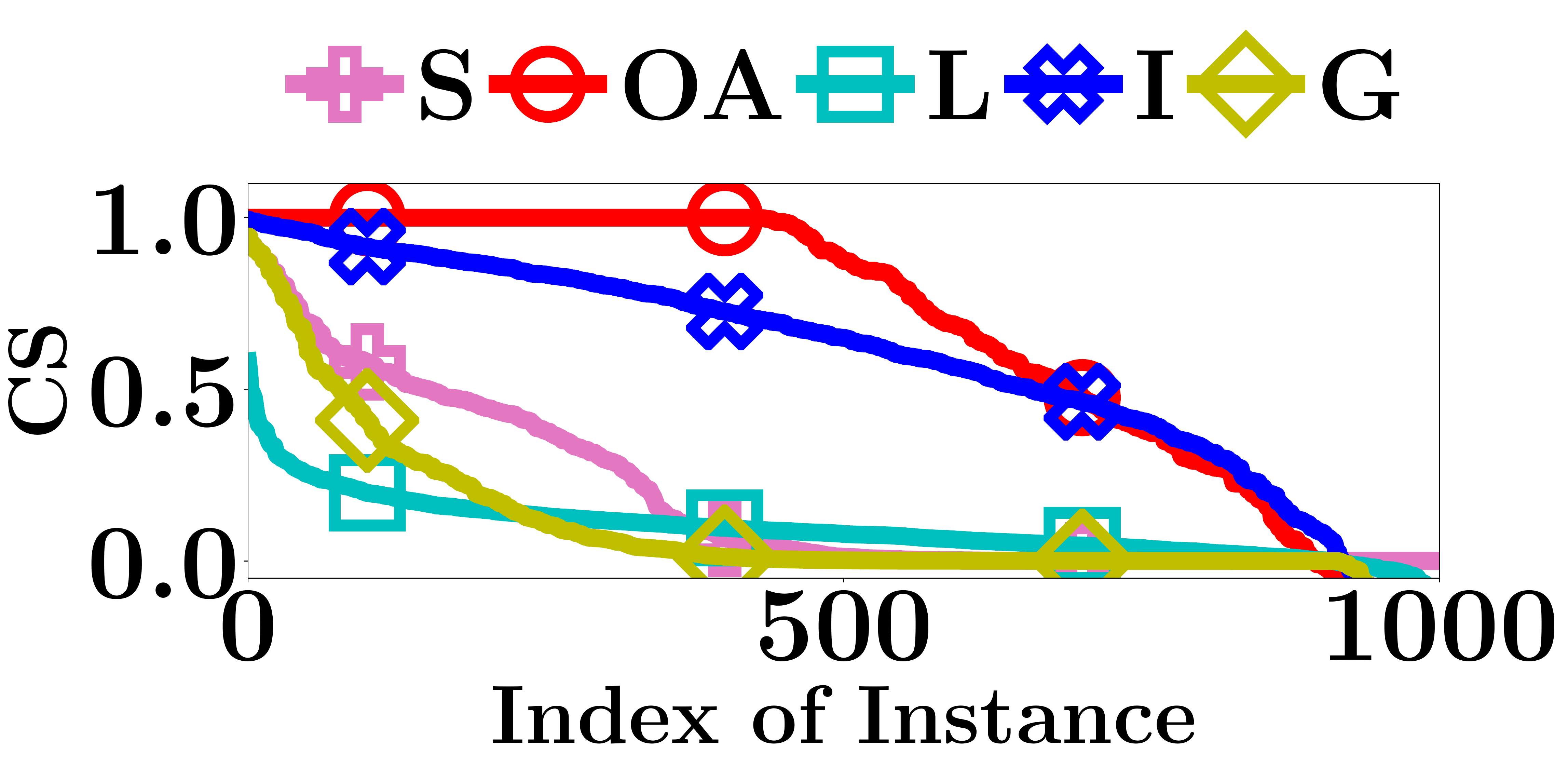}
}
\subfloat[MNIST (PLNN)]{
\includegraphics[page=1, width=\figwidthone\linewidth]{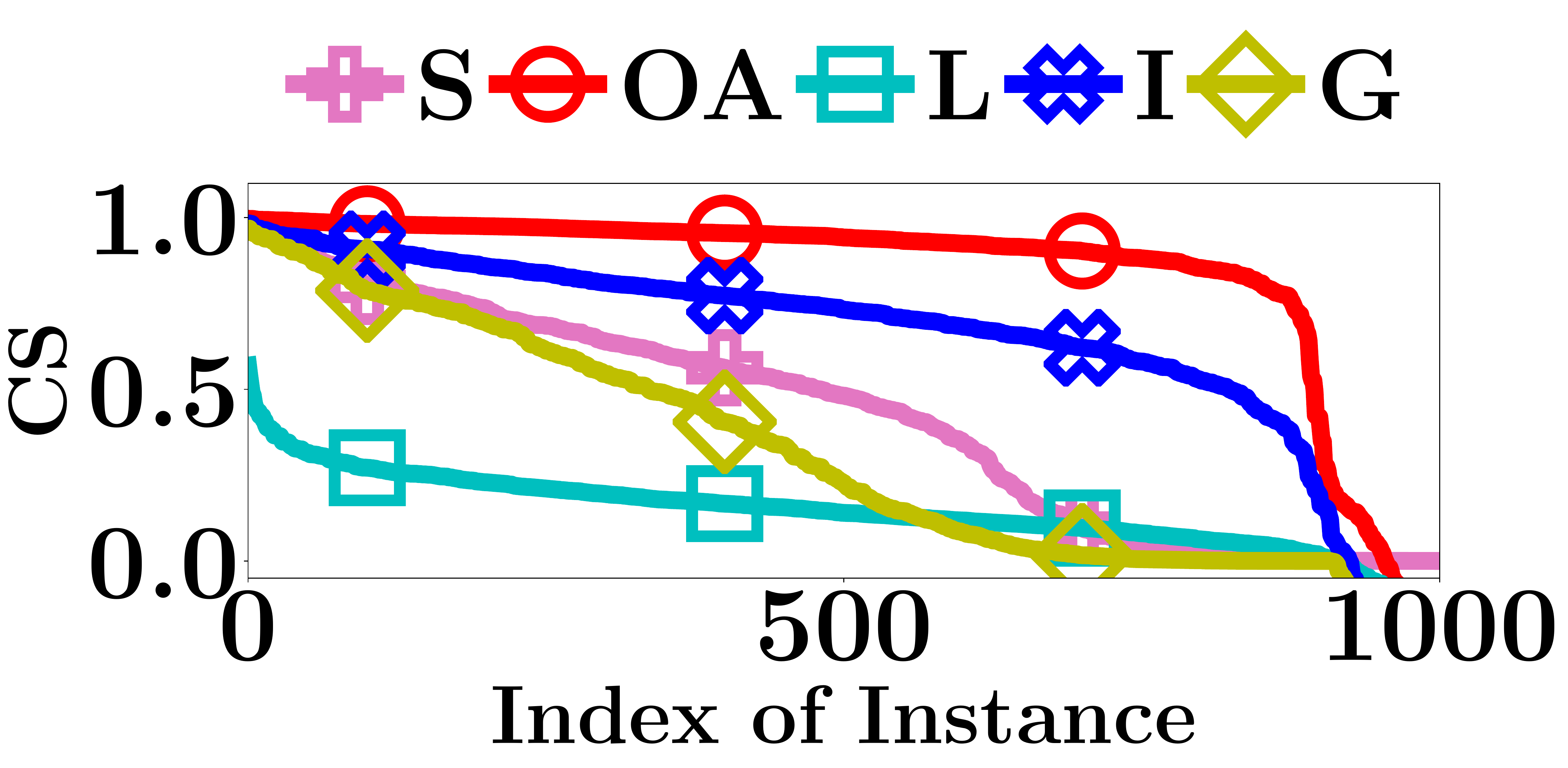}
}
\centering
\caption{The cosine similarity (CS for short) between the interpretations of each instance and its nearest neighbor. The results are separately sorted in the descending order of cosine similarity. ``S", ``OA", ``I", ``G", and ``L" have the same meaning as in Figure~\ref{fig:effect}. Each curve represents a method, and is plotted using 1000 data points.  We use different markers to make the curves more legible. }
\label{fig:cs}
\end{figure}

\subsection{Are the Interpretations Consistent?}
Consistent interpretation methods provide similar interpretations for similar input instances, and produce fewer contradictions between interpretations. Consistency is important.  For example, it is confusing if, for the instances in a locally linear region, the weights of their corresponding features are not the same, because those instances are classified by the same locally linear classifier of the PLM.

Using the same experiment settings as Chu \emph{et~al.}~\cite{chu2018exact}, 
we comprehensively analyze the consistency of the interpretations produced by \texttt{Saliency Maps}, \texttt{Integrated Gradient}, \texttt{Gradient * Input}, and \texttt{OpenAPI} by comparing the decision features of similar input instances.

Denote by $\mathbf{x}^0$ an input instance classified as class $c$, and by $\mathbf{x}^1$ the testing instance that is the nearest neighbour of $\mathbf{x}^0$ in Euclidean distance. 
\nop{The consistency of an interpretation on $\mathbf{x}^0$ is measured by the \textbf{Cosine Similarity} (CS) between the decision features of $\mathbf{x}^0$ and $\mathbf{x}^1$~\cite{chu2018exact}.}
For an interpretation method, we measure the interpretation consistency by \nop{the cosine similarity}the \textbf{Cosine Similarity} (\textbf{CS}) between the computed \nop{decision features}interpretations of $\mathbf{x}^0$ and $\mathbf{x}^1$.
Apparently, a larger CS indicates a better interpretation consistency.

The CS of all compared methods are evaluated on the testing data sets of FMNIST and MNIST.
As shown in Figure~\ref{fig:cs}, the interpretations given by \texttt{Integrated Gradient} are more consistent than the other two gradient based methods. \texttt{Integrated Gradient} smooths the differences between the interpretations for similar instances using the average partial derivatives of a set of instances to compute its interpretations. At the same time, the smooth operation also decreases the accuracy of its interpretations.
The CS of \texttt{OpenAPI} is better than all other methods on all PLMs and datasets. All instances contained in the same locally linear region have exactly the same decision features, and thus the CS of \texttt{OpenAPI} should always equal to 1 on those instances. As an input instance $\mathbf{x}^0$ and its nearest neighbor in the test set may not always belong to the same locally linear region, the CS of \texttt{OpenAPI} is not equal to 1 for all instances in our experiments. The poor performances of the baseline methods can be anticipated, since their interpretations rely on the gradients of the input instances, and they tend to provide distinct interpretations for individual instances.

In summary, the interpretation consistency of \texttt{OpenAPI} is significantly better than the other baseline methods.

\subsection{How Well Are the Perturbed Instances?}
\label{sec:sample}
The accuracy of all compared methods in computing $D_c$ of an input instance $\mathbf{x}^0$ largely depends on the \textbf{quality} of the set of sampled instances.
Here, the quality of a set of instances is good if they are contained in the same locally linear region as $\mathbf{x}^0$, and thus those instances have the same core parameters as $\mathbf{x}^0$ and significantly improve the accuracy in computing $D_c$. 

To comprehensively evaluate the performance of the compared methods in sampling a set of good instances, we measure the quality of the sampled instances by the following two metrics.

The \textbf{Region Difference (RD)} measures the consistency of the locally linear regions of the sampled instances. 
For any input instance $\mathbf{x}^0$, if all sampled instances are contained in the same locally linear region as $\mathbf{x}^0$, then $\text{RD}=0$; otherwise, $\text{RD}=1$.

\nop{
Denote by $D_c^i=\frac{1}{C-1} \sum_{c'=1}^C D_{c,c'}^{i}$ the average of the core parameters of an instance $\mathbf{x}^i, i\in\{0, \ldots, d+1\}$.
We define \textbf{Weight Difference (WD)} as 
\begin{equation}\nonumber
	\text{WD} = \frac{\sum_{i=1}^{d+1} || D_c^{0} - D_c^{i} ||_{1}}{d+1},
\end{equation}
to measure the average L1 distance between the $D_c^0$ of the input instance $\mathbf{x}^0$ and the $D_c^i$ of the $i$-th sampled instance $\mathbf{x}^i, i\in\{1, \ldots, d+1\}$.
}

The \textbf{Weight Difference (WD)} is defined as 
\begin{equation}\nonumber
	\text{WD} = \frac{\sum_{c'=1}^C\sum_{i=1}^{|S|} || D_{c,c'}^{0} - D_{c,c'}^{i} ||_{1}}{(C-1)|S|},
\end{equation}
which measures the average L1 distance between $D_{c,c'}^0$ of the input instance $\mathbf{x}^0$ and $D_{c,c'}^i$ of the $i$-th instance $\mathbf{x}^i$ in the set of sampled instances $S=\{\mathbf{x}^1, \ldots, \mathbf{x}^{m}\}$.

Apparently, a RD that is equal to $0$ indicates a perfect consistency among the locally linear regions of all sampled instances.
A small WD means a high similarity between the core parameters of $\mathbf{x}^0$ and the sampled instances.
If RD and WD are both small, the quality of the sampled instances is good, and $D_c$ of $\mathbf{x}^0$ can be accurately computed.

\nop{We evaluate RD and WD of \texttt{ZOO} and \texttt{OpenAPI} on the testing data sets FMNIST-1 and FMNIST-2.}
We evaluate RD and WD of \texttt{ZOO}, \texttt{Linear Regression LIME}, \texttt{Ridge Regression LIME}, the naive method,  and \texttt{OpenAPI} on the testing data sets of FMNIST and MNIST.
For each data set, we use every testing instance as the input instance $\mathbf{x}^0$ once, and evaluate RD and WD of the corresponding set of sampled instances. Figures~\ref{fig:wsrd} and~\ref{fig:wswd}, respectively, show the average RD and WD of all testing instances.
\nop{The performance of \texttt{LIME} is not evaluated, because the published code of \texttt{LIME} does not return any sampled instances.}

The performance of the baseline methods in RD and WD relies heavily on the heuristic perturbation distance $h$.
Since there is no effective method to set $h$, we evaluate the performance of the baseline methods with respect to a wide range of $h$. Specifically,  we test $h=10^{-2}$, $10^{-4}$, and $10^{-8}$.

As shown in Figure~\ref{fig:wsrd}, the average RD of the baseline methods increases when $h$ increases.  The results verify our claim in Section~\ref{sec:openapi} that a smaller hypercube is more likely to be contained in the locally linear region of an input instance.

\begin{figure}[t]
\centering

\subfloat[FMNIST (LMT)]{
\includegraphics[page=2, width=\figwidthone\linewidth]{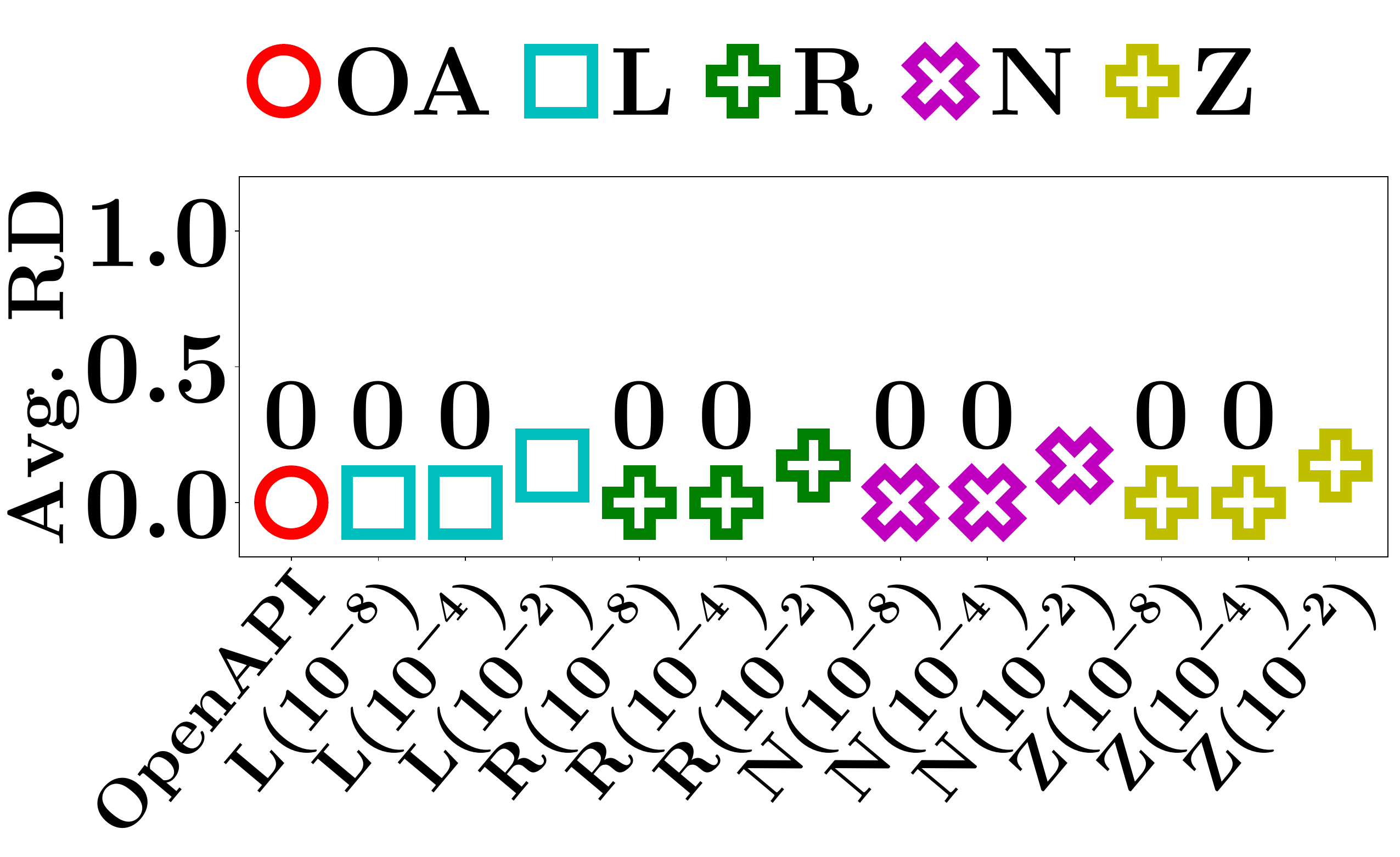}
}
\subfloat[FMNIST (PLNN)]{
\includegraphics[page=2, width=\figwidthone\linewidth]{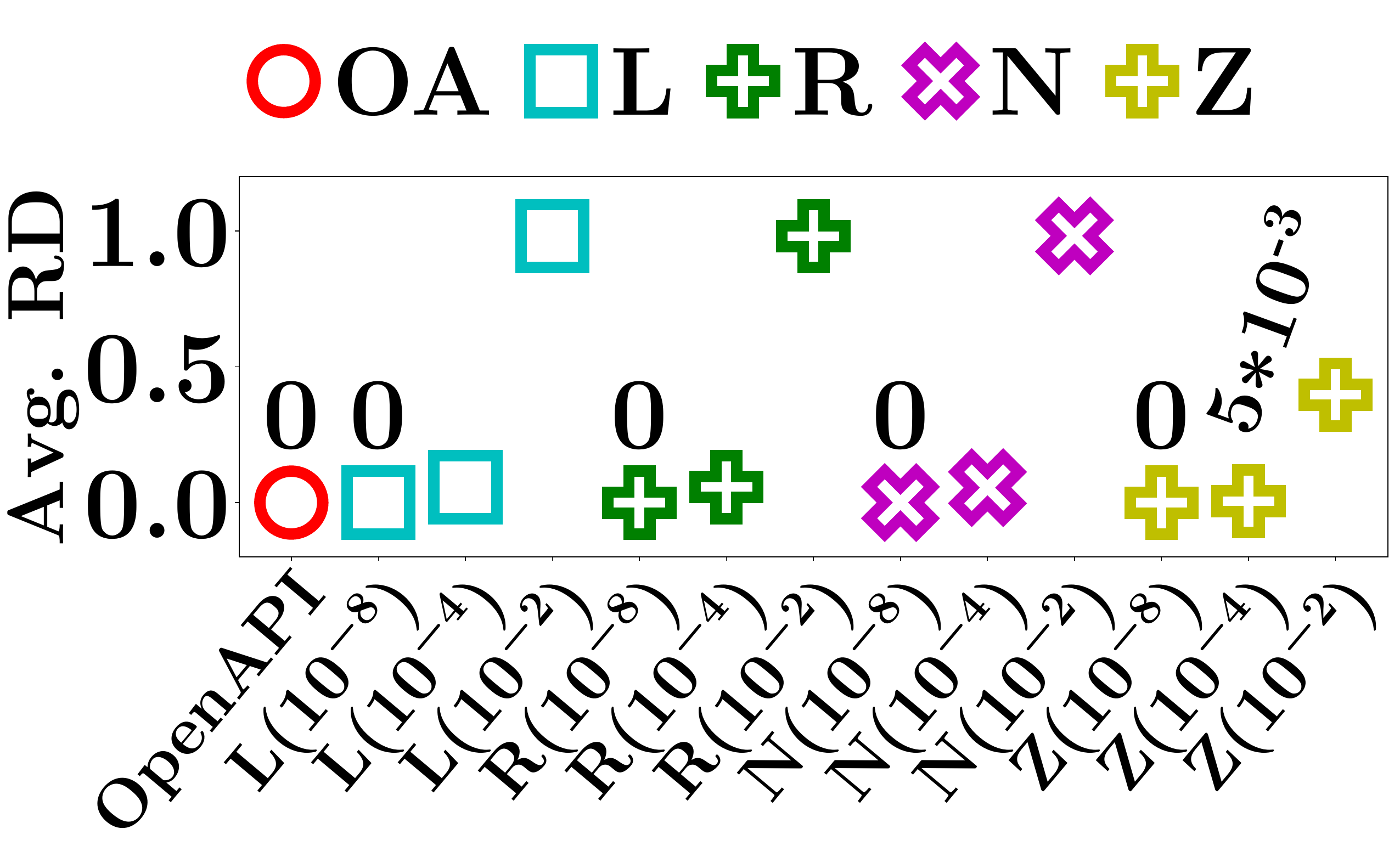}
}

\qquad
\subfloat[MNIST (LMT)]{
\includegraphics[page=2, width=\figwidthone\linewidth]{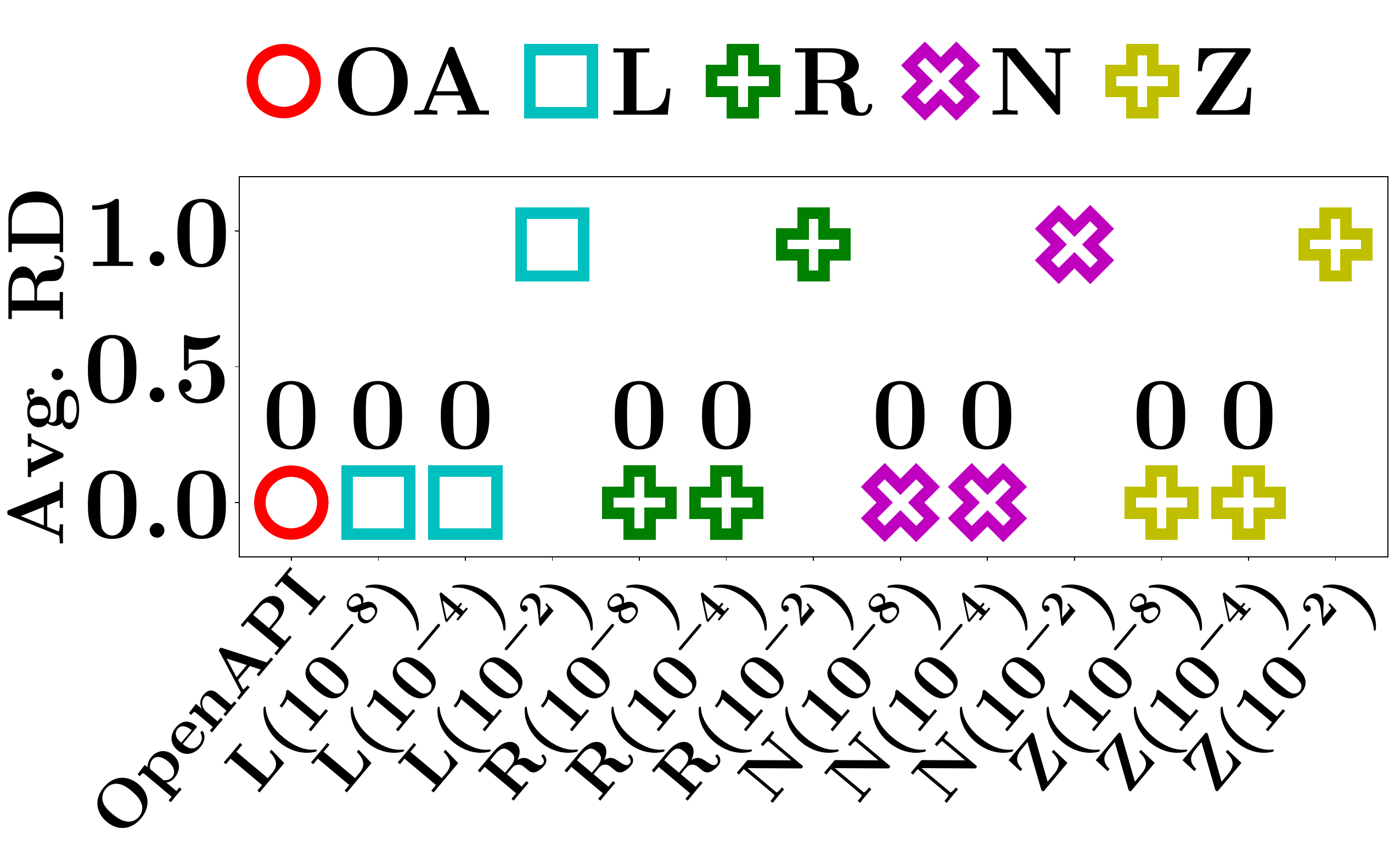}
}
\subfloat[MNIST (PLNN)]{
\includegraphics[page=2, width=\figwidthone\linewidth]{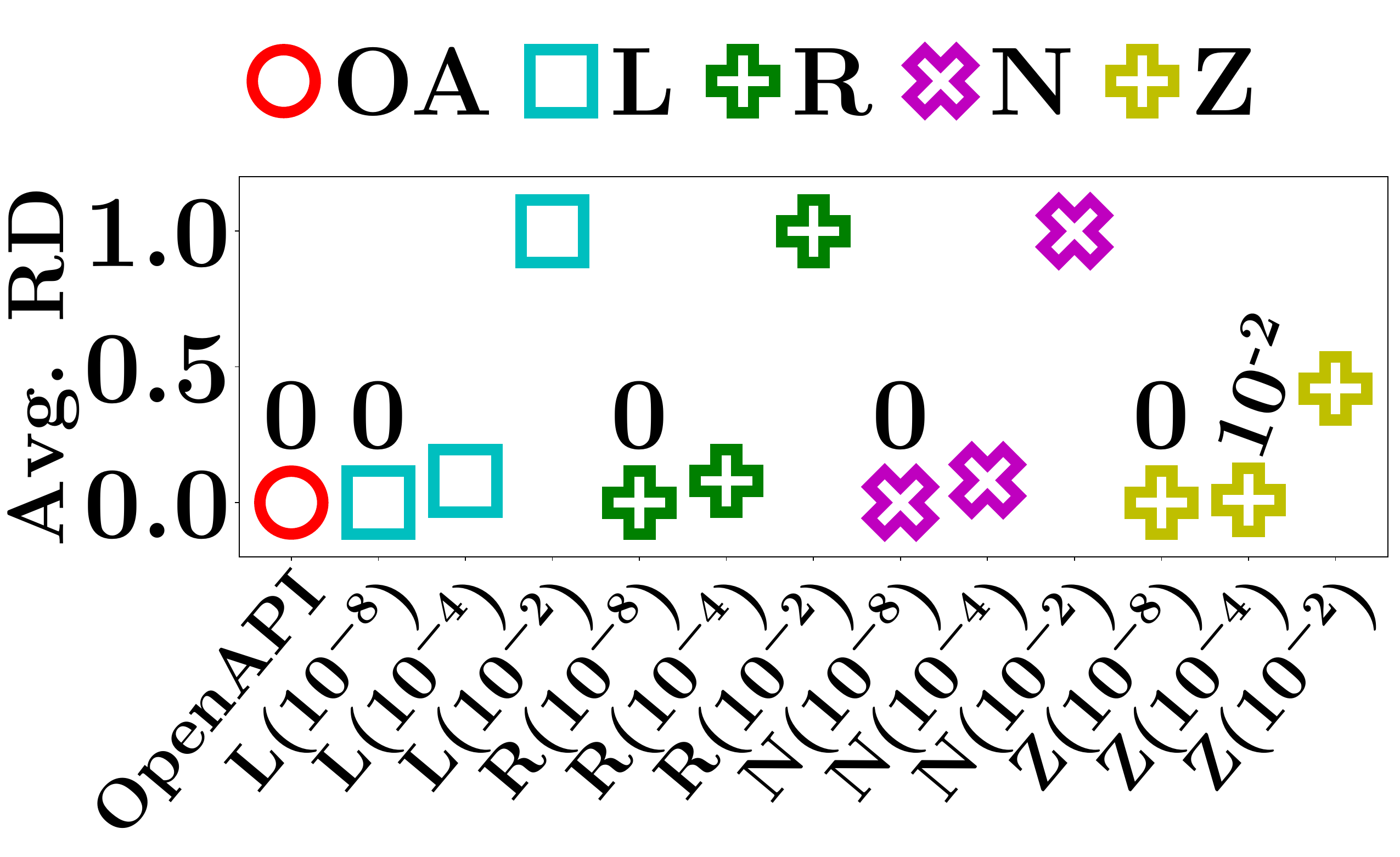}
}


\centering

\caption{The average RD of all methods. $N(h)$, $Z(h)$, $L(h)$, and $R(h)$ are the performance measures of the naive method, \texttt{ZOO}, \texttt{Linear Regression LIME}, and \texttt{Ridge Regression LIME} with respect to perturbation distance $h$, respectively. \nop{$Z(h)$ is the performance of \texttt{ZOO} with respect to perturbation distance $h$.} We give the average RD values on top of some ticks for the ease of reading.}
\label{fig:wsrd}
\end{figure}

Since the RD of\nop{\texttt{ZOO}} the baseline methods drops to 0 when $h$ is small, it is appealing to ask whether we can fix $h$ to a small value such that\nop{\texttt{ZOO}} the baseline methods can always find a good sample of instances.
Unfortunately, this is impossible.  The volume of locally linear regions varies significantly for different PLMs.
For example, 
as shown in Figure~\ref{fig:wsrd}, when the perturbation distance $h=10^{-4}$, the RD of \texttt{ZOO} is $0$ for LMT, but it is $0.005$ and $0.01$ for PLNN on FMNIST and MNIST, respectively.
Thus, $h=10^{-4}$ is good for LMT, but not good enough for PLNN.
One may argue that conservatively $h$ can take an extremely small value in the hope that it works for both LMT and PLNN.
However, since the number of locally linear regions of a PLNN is exponential with respect to the number of hidden units~\cite{montufar2014number, chu2018exact, pascanu2013number}, the volume of some locally linear regions of a large PLNN can be arbitrarily close to zero. For any fixed value of $h$, one can always construct a counter example that $h$ is still too big for PLNN.
Even for the same PLM, the good perturbation distance may still vary significantly for different input instances, and can be arbitrarily small.

\nop{
Therefore, it is difficult for a heuristically fixed value of $h$ to work well for all input instances.
}

\nop{
As demonstrated by the WD performance in Figure~\ref{fig:wswd}, for any $h\geq 10^{-2}$, there is a big difference between the maximum and minimum values of WD.

This is because the distance between an input instance and the boundary of a locally linear region varies significantly for different input instances.

an input instance $\mathbf{x}^0$ can be infinitely close to the boundary of a locally linear region, 
}

\begin{figure}[t]
\centering

\subfloat[FMNIST (LMT)]{
\includegraphics[page=1, width=\figwidthone\linewidth]{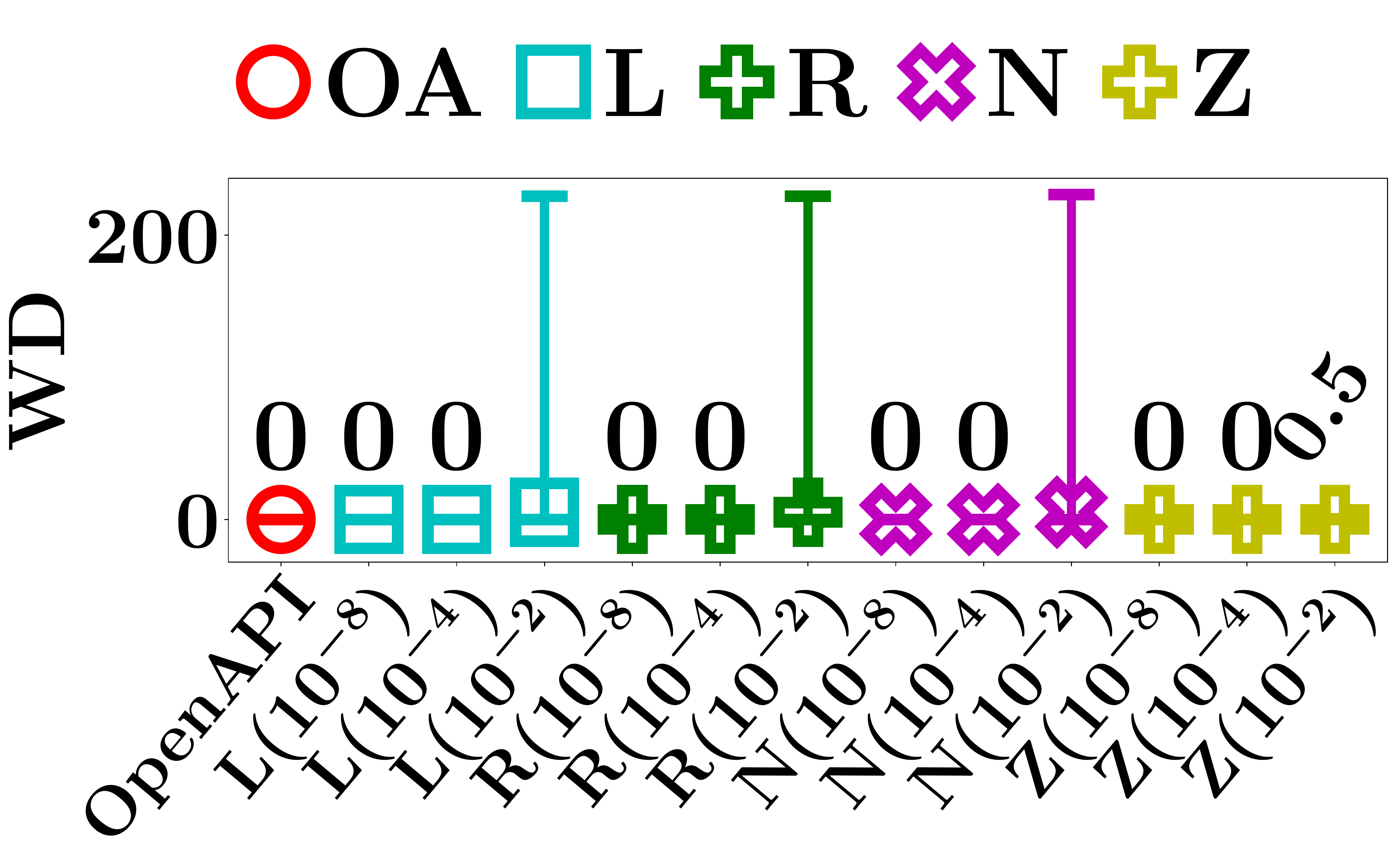}
}
\subfloat[FMNIST (PLNN)]{
\includegraphics[page=1, width=\figwidthone\linewidth]{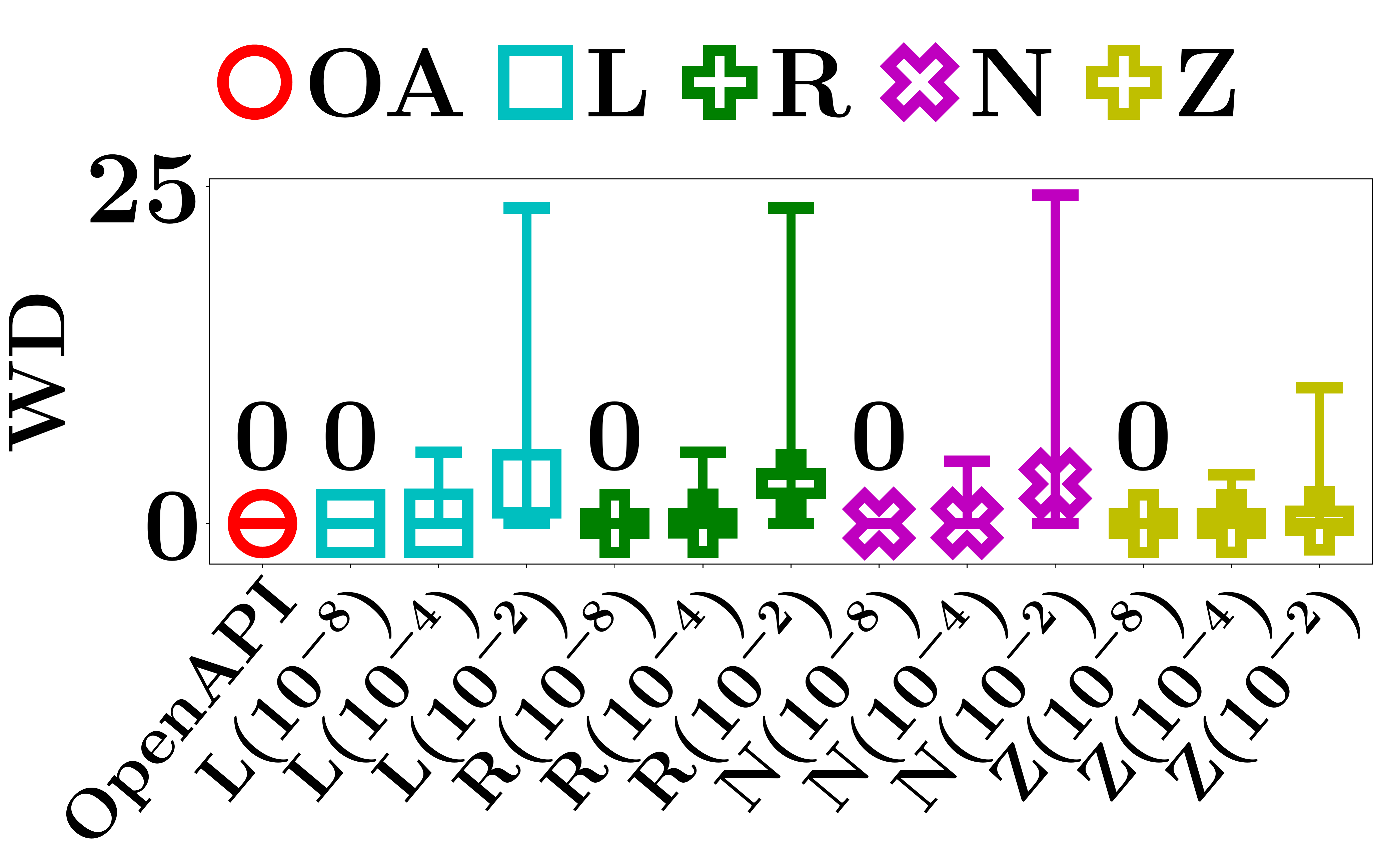}
}

\qquad
\subfloat[MNIST (LMT)]{
\includegraphics[page=1, width=\figwidthone\linewidth]{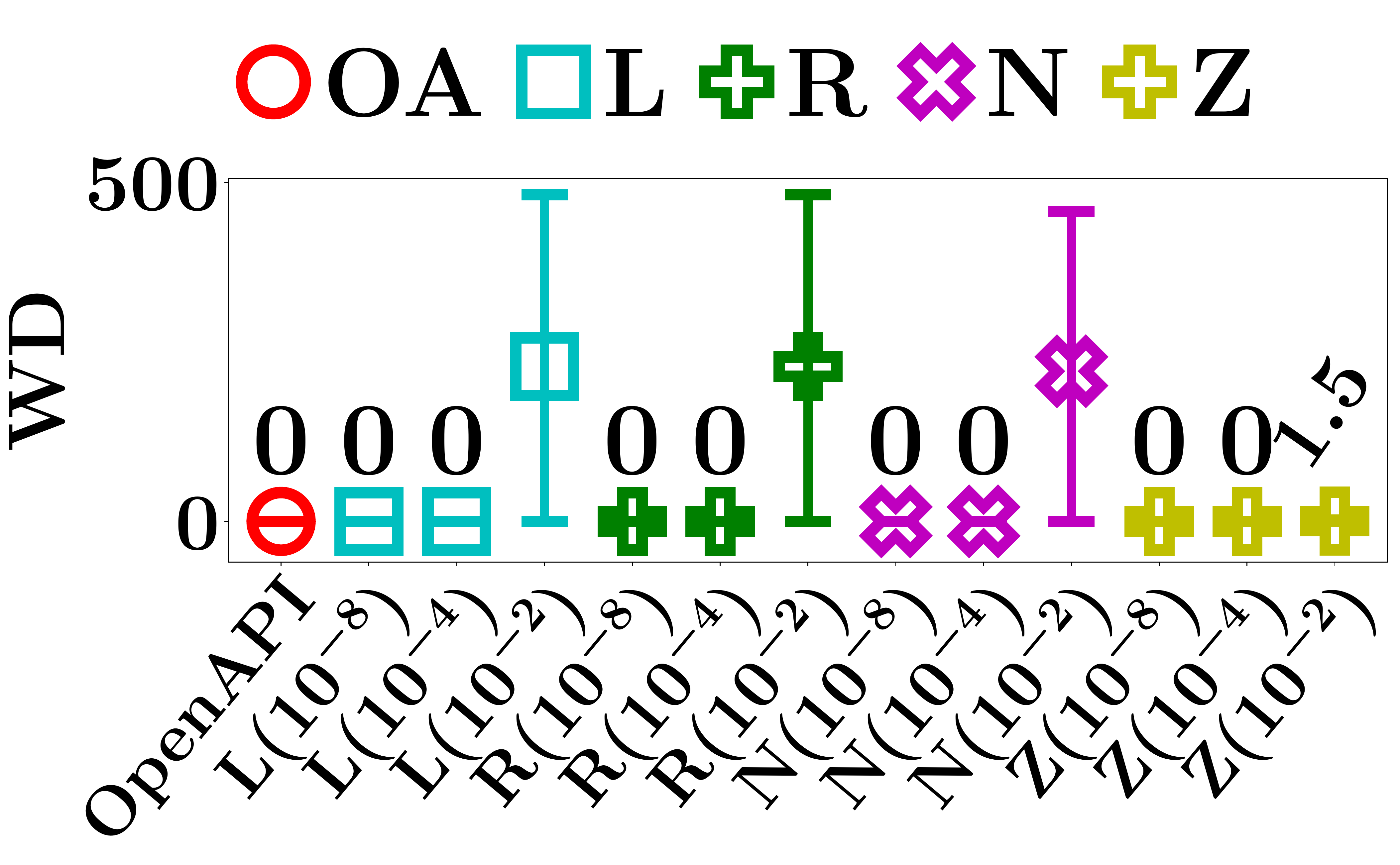}
}
\subfloat[MNIST (PLNN)]{
\includegraphics[page=1, width=\figwidthone\linewidth]{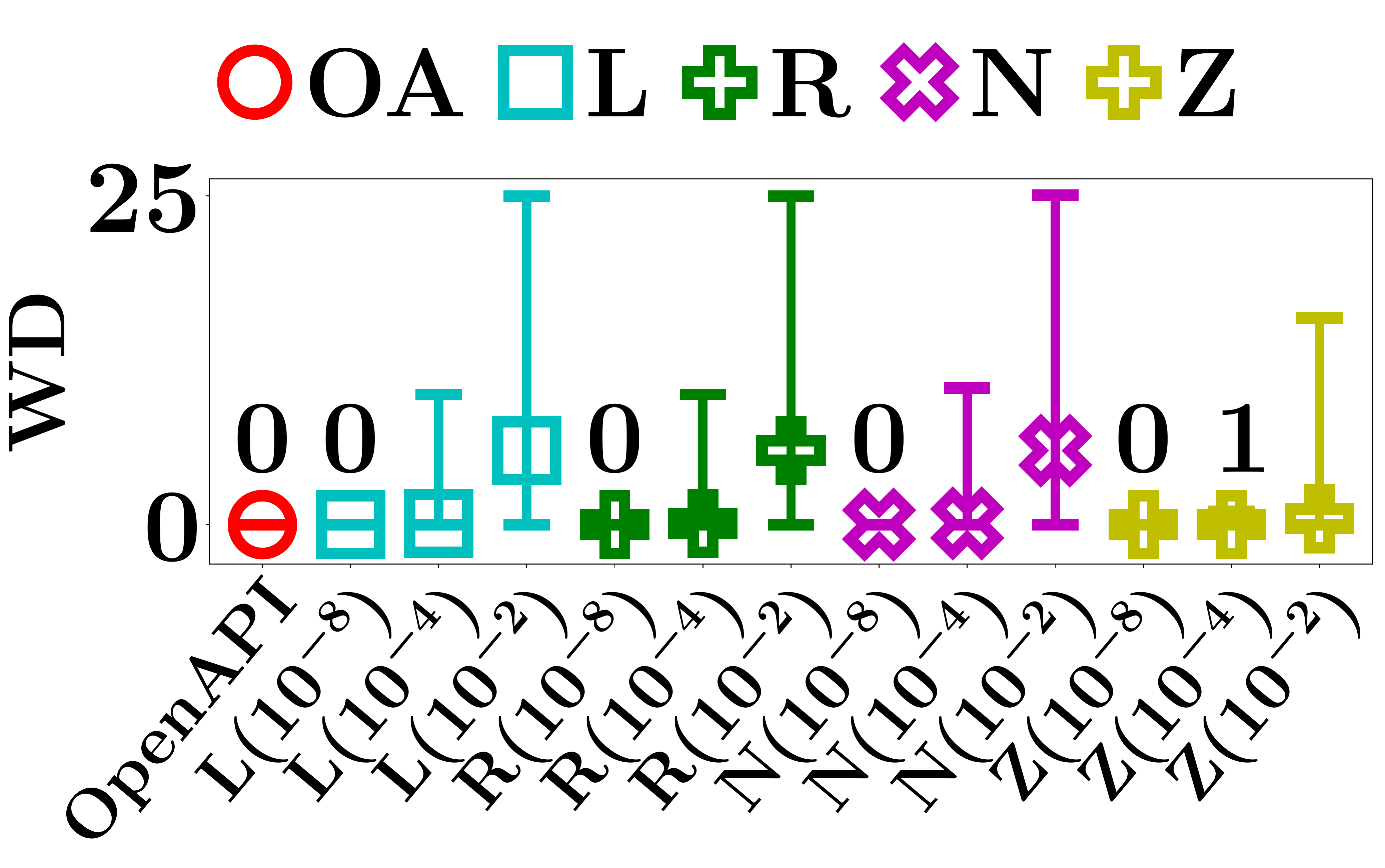}
}






\centering

\caption{WD of all methods. 
The upper and lower ends of an error bar show the maximum and minimum WD, respectively, for all testing instances.
A marker represents the mean of WD in the corresponding bar. 
We give the maximum of WD values on top of some ticks for the ease of reading. 
$N(h)$, $Z(h)$, $L(h)$, and $R(h)$ have the same meanings as in Figure~\ref{fig:wsrd}.
}
\label{fig:wswd}
\end{figure}

Recall that we can only access the API of a PLM, we have no knowledge about the size of the locally linear regions of the PLM. This makes it even harder to initialize an optimal value of perturbation distance $h$ that works universally on all PLMs and input instances. 
A much better method is to sample a set of good instances in an adaptive manner, just as what \texttt{OpenAPI} does.

As shown in Figures~\ref{fig:wsrd} and~\ref{fig:wswd}, the average RD and WD of \texttt{OpenAPI} are\nop{very small} $0$ on all data sets. This demonstrates the superior capability of \texttt{OpenAPI} in adaptively sampling a set of good instances.

\begin{figure}[t]
\centering

\subfloat[FMNIST (LMT)]{
\includegraphics[page=1, width=\figwidthone\linewidth]{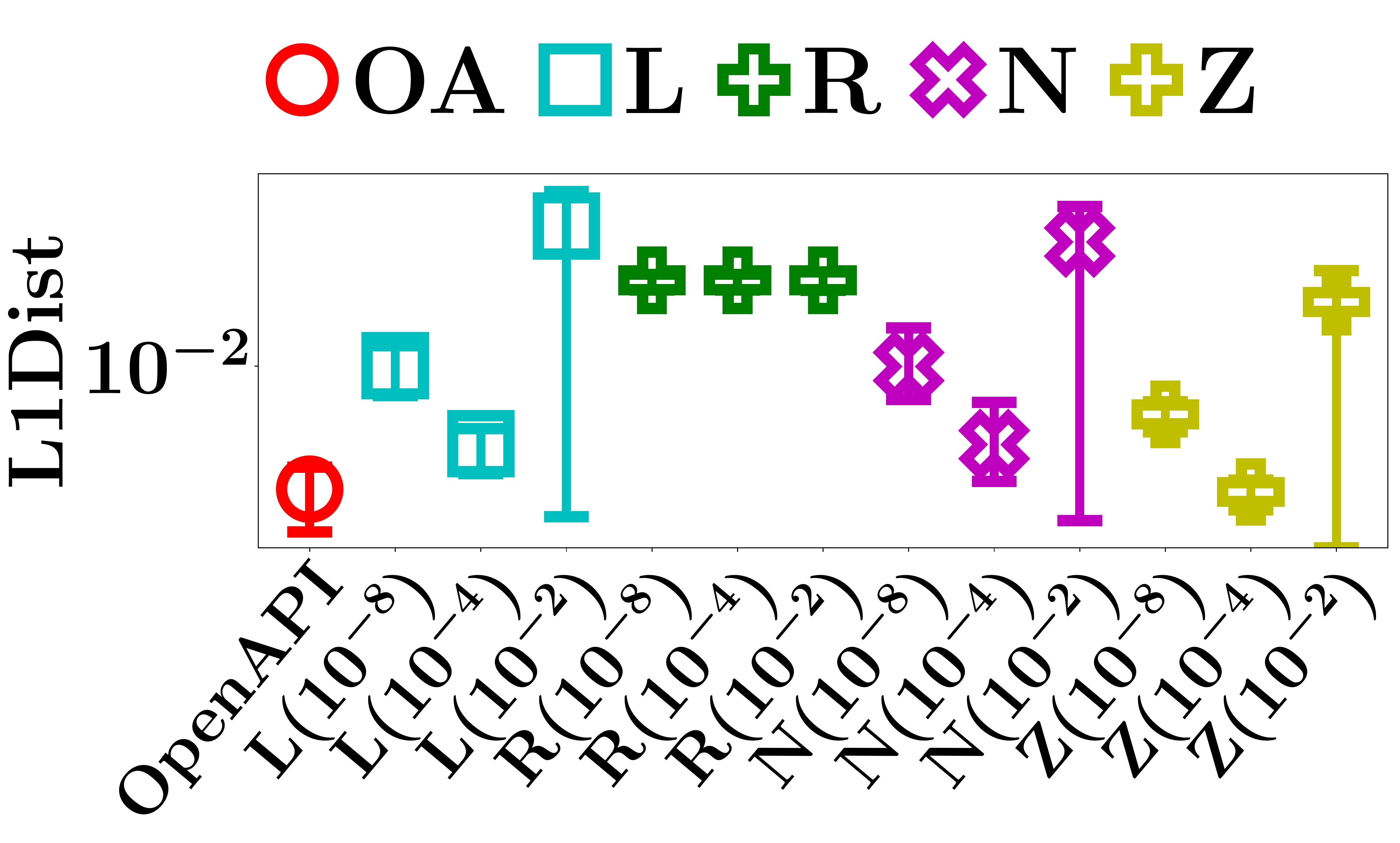}
}
\subfloat[FMNIST (PLNN)]{
\includegraphics[page=1, width=\figwidthone\linewidth]{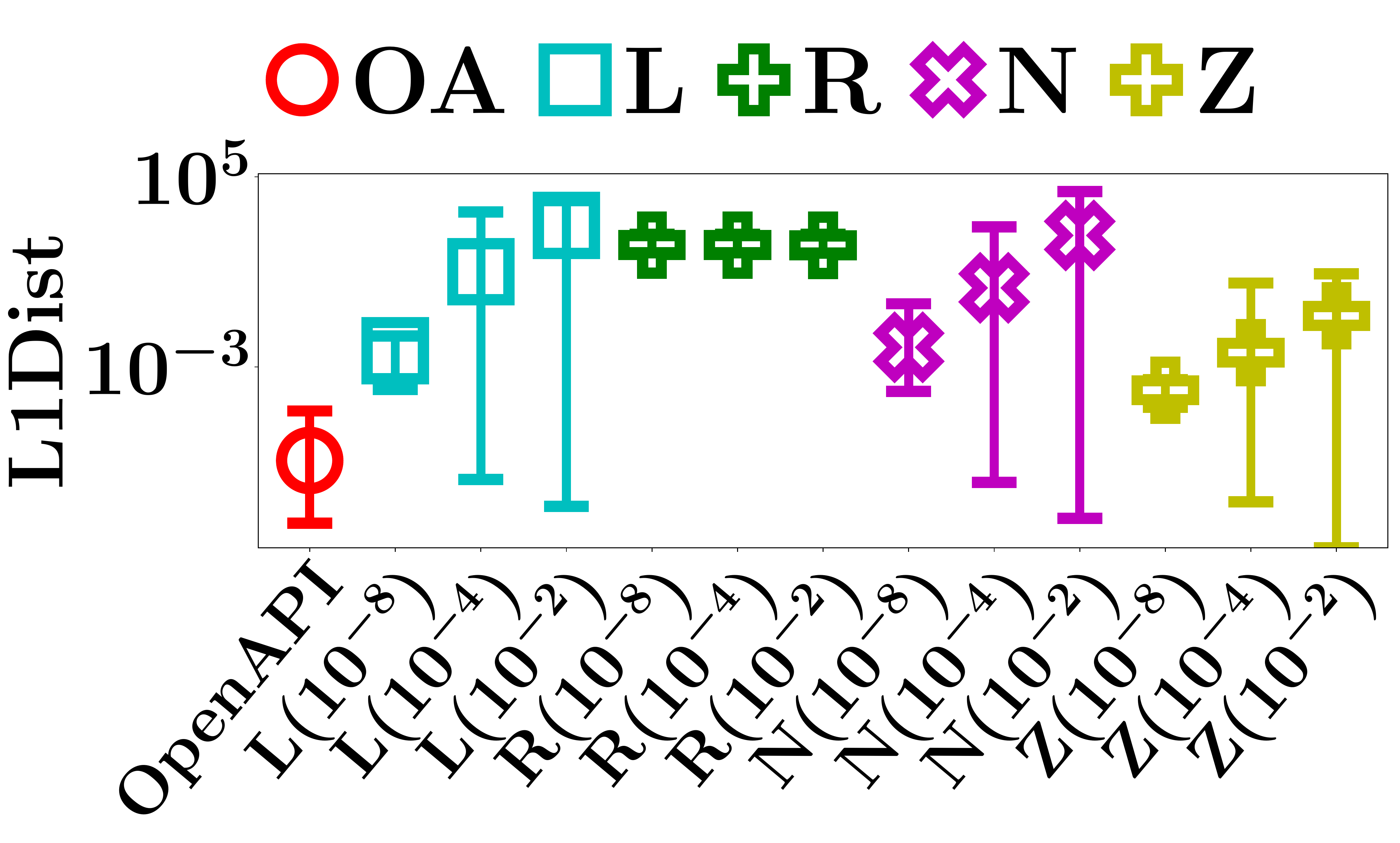}
}

\qquad
\subfloat[MNIST (LMT)]{
\includegraphics[page=1, width=\figwidthone\linewidth]{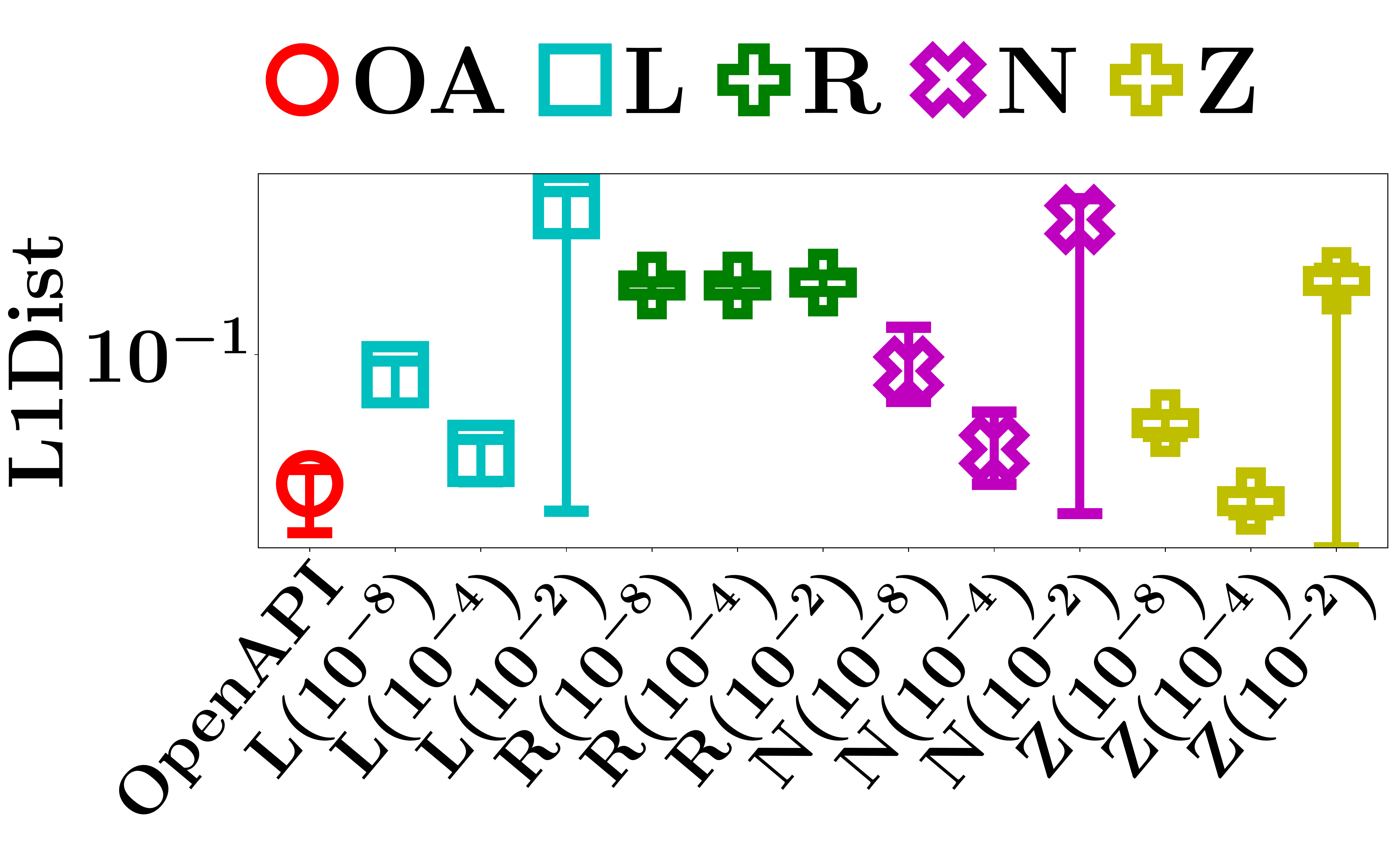}
}
\subfloat[MNIST (PLNN)]{
\includegraphics[page=1, width=\figwidthone\linewidth]{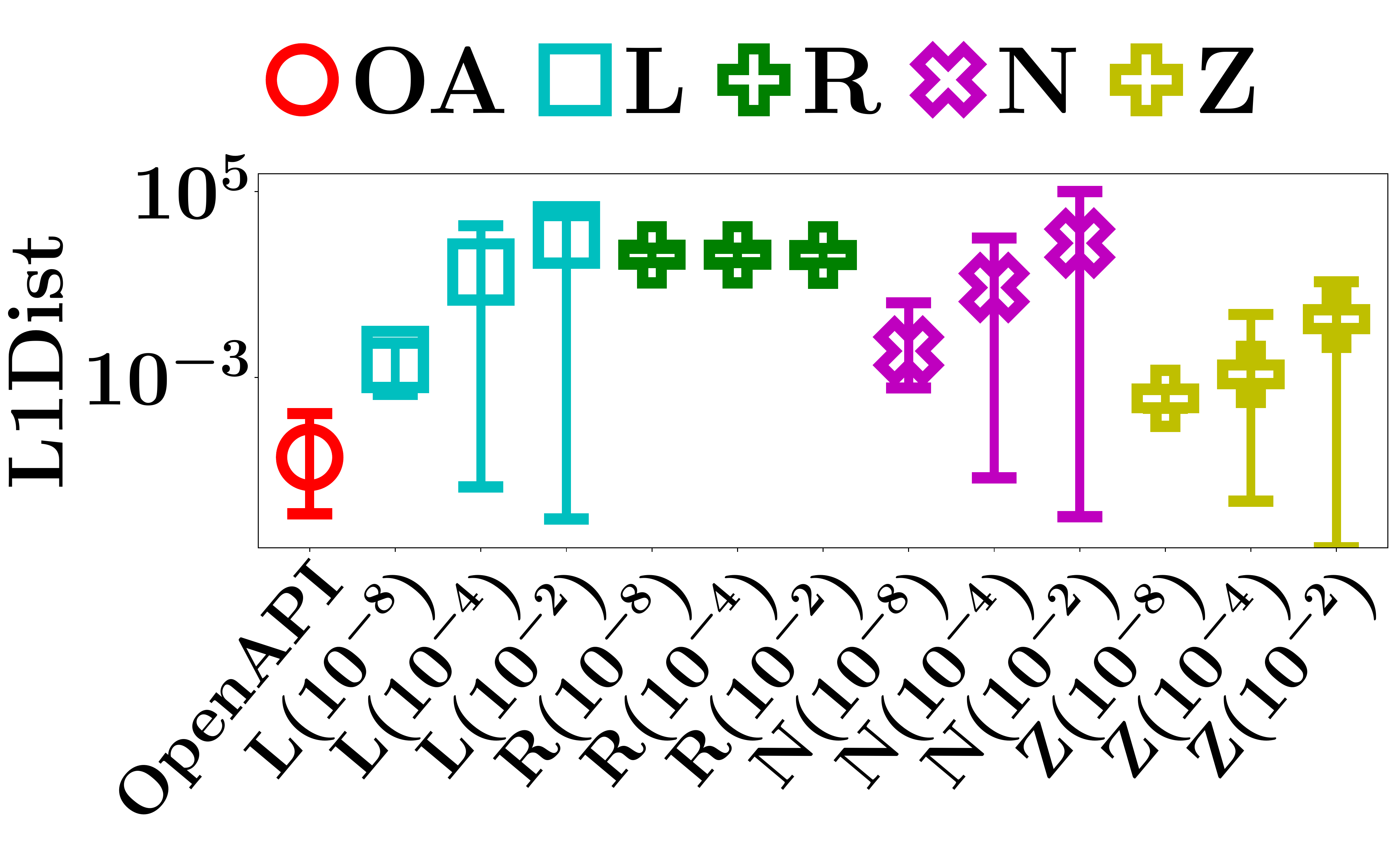}
}






\centering

\caption{
 L1Dist of all methods.
The upper and lower ends of an error bar show the maximum and minimum L1Dist, respectively, for all testing instances.
A marker represents the mean of L1Dist. 
The L1Dist of the methods are plotted in logarithmic scale. $N(h)$, $Z(h)$, $L(h)$, and $R(h)$ have the same meaning as in Figure~\ref{fig:wsrd}.}
\label{fig:l1dist}
\end{figure}

\subsection{Are the Interpretations Exact?}

In this subsection, we systematically study the exactness of interpretations by comparing the ground truth of the decision features of a PLM with the decision features identified by \texttt{ZOO}, \texttt{Linear Regression LIME}, \texttt{Ridge Regression LIME}, the naive method, and \texttt{OpenAPI}.


Denote by ${D}_c$ the ground truth of decision features of a PLM in classifying an input instance as class $c$, and by $D^*_c$ the decision features computed by an interpretation method. 
We measure the exactness of an interpretation by  \textbf{L1Dist}, the L1 distance between $D_c$ and $D_c^*$.
Obviously, a smaller L1Dist indicates a higher exactness of an interpretation.

We evaluate L1Dist of the four baseline methods and \texttt{OpenAPI} on the testing data sets. \nop{FMNIST-1 and FMNIST-2}
For each data set, we use every testing instance as the input instance, and evaluate L1Dist of the interpretations.
The average, minimum and maximum L1Dist of all testing instances are reported in Figure~\ref{fig:l1dist}.

The large L1Dist of \texttt{Ridge Regression LIME} on all datasets indicates that $D_c^*$ computed by the method is significantly different from $D_c$ of the PLMs. By carefully investigating the learned classifiers of \texttt{Ridge Regression LIME}, we find that when the perturbed distances are very small, the linear function used to approximate the predictions always converges to a constant function that always outputs the expected value of the predictions. The poor exactness of \texttt{Ridge Regression LIME} is mainly caused by the mis-selected approximate model. As a comparison, \texttt{Linear Regression LIME}, which has no constraints on its coefficient matrix, performs much better than its counterpart with ridge regression.

\nop{L1Dist of \texttt{LIME} is large on all data sets, which means $D_c^*$ computed by \texttt{LIME} is significantly different from $D_c$ of the PLMs. The poor exactness of \texttt{LIME} is anticipated, because \texttt{LIME} heuristically applies large perturbations to sample a set of low quality instances, which may be mostly not contained in the same locally linear region of the input instance. As comparison, the }

\nop{The minimum and maximum L1Dist of \texttt{ZOO} and \mc{the naive method} with $h=10^{-2}$ are different significantly. As shown in Figure~\ref{fig:wsrd}, for some testing instances, the perturbation distance $h=10^{-2}$ is already small enough that all perturbed instances are from the same locally linear regions. Therefore, for those instances, their decision features can be accurately computed. However, for the inputs whose perturbed instances are not from the same local region, the computed decision features could be very different from the true ones. The experiment results are consistent with Theorem~\ref{thm:impossible} in Section~\ref{method:naive}.}

The L1Dist of the other baseline methods increases significantly when the perturbation distance $h$ becomes larger than a critical value. Since a smaller $h$ leads to a better quality of the sampled instances, it usually increases the accuracy of most baseline methods in computing $D_c^*$.
However, as discussed in Section~\ref{sec:sample}, the critical value of $h$ varies significantly for different models and instances, thus it is impossible to find a golden value of $h$ that always achieves the best L1Dist in computing the decision features of all models and instances.

We can also see that when $h$ becomes extremely small, L1Dist increases. The reason is that all methods suffer from the classical problem of softmax saturation. When an input instance $\mathbf{x}^0$ is classified with a probability extremely close to $1$ and the perturbed distance $h$ becomes extremely small, the PLMs have almost the same predictions on the perturbed instances and the original instance. As a result, the computation of the decision features becomes unstable, which goes beyond the limited precision of Python in stably manipulating floating point numbers. Also, extremely small perturbations lead to linear equation systems with large condition numbers, which are hard to solve numerically. Due to the above reasons, extremely small perturbations hurt the exactness of all methods.

The computation of the decision features becomes unstable due to two reasons. 

\nop{We can also see that when $h$ becomes extremely small, L1Dist of \texttt{ZOO} becomes very large.
The reason is that \texttt{ZOO} computes the decision features by  the symmetric difference quotient~\cite{lax2014calculus}, which requires to compute the reciprocal of $h$. When $h$ becomes extremely small, the computation of the reciprocal of $h$ becomes unstable, which goes beyond the limited precision of python in stably manipulating floating point numbers.  This is a subtle issue in \texttt{ZOO}.}

In contrast, since \texttt{OpenAPI} is able to find the exact decision features of a PLM with probability $1$, it achieves the best L1Dist performance on all data sets. In addition, as shown in Table \ref{tab:avg_iter}, \texttt{OpenAPI} can find the exact interpretations with only a small number of iterations.

\begin{table}[]
    \centering
    \begin{tabular}{|c|c|c|}
    \hline
        Data Set & FMNIST & MNIST\\
    \hline
        LMT & 6.0 & 8.6 \\
    \hline
        PLNN & 10.3 & 10.8 \\
    \hline
    \end{tabular}
    \caption{The average number of iterations of \texttt{OpenAPI} to compute the interpretations for the models}
    \label{tab:avg_iter}
\end{table}

\section{Conclusions}
\label{sec:con}
In this paper, we tackle the challenge of interpreting a PLM hidden behind an API.  In this problem, neither model parameters nor training data are available.
By finding the closed form solutions to a set of overdetermined equation systems constructed using a small set of sampled instances, we develop \texttt{OpenAPI}, a simple yet effective and efficient method accurately identifying the decision features of a PLM with probability $1$.
We report extensive experiments demonstrating the superior performance of \texttt{OpenAPI} in producing exact and consistent interpretations.
As future work, we will extend our work to reverse engineer PLMs hidden behind  APIs.


%
\bibliographystyle{abbrv}
\bibliography{OpenAPI}

\end{document}